\documentclass{article}

    \PassOptionsToPackage{numbers, compress}{natbib}
    \usepackage[final]{neurips_2024}

\usepackage[utf8]{inputenc} % allow utf-8 input
\usepackage[T1]{fontenc}    % use 8-bit T1 fonts
\usepackage{hyperref}       % hyperlinks
\usepackage{url}            % simple URL typesetting
\usepackage{booktabs}       % professional-quality tables
\usepackage{amsfonts}       % blackboard math symbols
\usepackage{nicefrac}       % compact symbols for 1/2, etc.
\usepackage{microtype}      % microtypography
\usepackage{xcolor}         % colors

\usepackage{times}
\usepackage{enumitem}
\usepackage{hyperref}
\usepackage{natbib}
\usepackage{nicefrac}
\usepackage{empheq}
\usepackage{microtype}
\usepackage{graphicx}
\usepackage{subfigure}
\usepackage{mdframed}
\usepackage{amsmath}
\usepackage{amssymb}
\usepackage{mathtools}
\usepackage{amsthm}
\newmdtheoremenv{boxedtheorem}{Theorem}
\usepackage{algorithmic}
\usepackage[capitalize,noabbrev]{cleveref}
\usepackage{algorithm}
\usepackage{multirow}
\usepackage{bm}
\usepackage{bbm}
\usepackage[noorphans,vskip=0.5ex]{quoting}
\usepackage{listings}

\definecolor{codegreen}{rgb}{0,0.6,0}
\definecolor{codegray}{rgb}{0.5,0.5,0.5}
\definecolor{codepurple}{rgb}{0.58,0,0.82}
\definecolor{backcolour}{rgb}{0.95,0.95,0.92}
\lstdefinestyle{mystyle}{
    backgroundcolor=\color{backcolour},   
    commentstyle=\color{codegray},
    keywordstyle=\color{codegreen},
    numberstyle=\tiny\color{codegray},
    stringstyle=\color{codepurple},
    basicstyle=\ttfamily\footnotesize,
    breakatwhitespace=false,         
    breaklines=true,                 
    captionpos=b,                    
    keepspaces=true,                 
    numbers=left,                    
    numbersep=5pt,                  
    showspaces=false,                
    showstringspaces=false,
    showtabs=false,                  
    tabsize=2
}

\theoremstyle{plain}
\newtheorem{theorem}{Theorem}[section]

\newtheorem{lemma}[theorem]{Lemma}

\theoremstyle{definition}

\theoremstyle{remark}

\usepackage[textsize=tiny]{todonotes}

\newcommand{\test}{\mathrm{test}}

\newcommand{\train}{\mathrm{train}}
\newcommand{\val}{\mathrm{val}}

\newcommand{\sumi}{\sum_{i=1}^m}
\newcommand{\sumj}{\sum_{j=1}^n}

\newcommand{\Var}{\mathrm{Var}}
\newcommand{\KL}{\mathrm{KL}}

\newcommand{\R}{\mathbb{R}}
\newcommand{\PP}{\mathbb{P}}
\newcommand{\E}{\mathbb{E}}

\newcommand{\D}{\mathcal{D}}
\newcommand{\cL}{\mathcal{L}}

\newcommand{\X}{\mathcal{X}}

\title{DAVED: Data Acquisition via Experimental Design for Data Markets}

\author{
  Charles Lu\\
  MIT \\
  \And
  Baihe Huang \\
  UC Berkeley\\ 
  \And
  Sai Praneeth Karimireddy \\
  USC, UC Berkeley\\
  \And
  Praneeth Vepakomma \\
  MIT, MBZUAI \\
  \And
  Michael I. Jordan \\
  UC Berkeley\\ 
  \And
  Ramesh Raskar \\
  MIT \\
}
\begin{document}

\maketitle

\begin{abstract}
    The acquisition of training data is crucial for machine learning applications. 
    Data markets can increase the supply of data, particularly in data-scarce domains such as healthcare, by incentivizing potential data providers to join the market.
    A major challenge for a data buyer in such a market is choosing the most valuable data points from a data seller. 
    Unlike prior work in data valuation, which assumes centralized data access, we propose a federated approach to the data acquisition problem that is inspired by linear experimental design. 
    Our proposed data acquisition method achieves lower prediction error without requiring labeled validation data and can be optimized in a fast and federated procedure.
    The key insight of our work is that a method that directly estimates the benefit of acquiring data for test set prediction is particularly compatible with a decentralized market setting.
\end{abstract}

\section{Introduction}\label{sec:intro}
Major breakthroughs in machine learning have only been possible due to the availability of massive amounts of training data. Such collections are still, however, mostly in the province of large companies, and it is an important agenda item for machine learning researchers to develop mechanisms that allow greater data access among smaller players. 
A related point is that many data owners have become resistant to having their data simply collected without their consent or without their participation in the fruits of predictive modeling, resulting in legal challenges against prominent AI companies.
% ~\footnote{\footnotesize See \url{https://stablediffusionlitigation.com}, \url{https://githubcopilotlitigation.com}, and \url{https://www.nytimes.com/2023/12/27/business/media/new-york-times-open-ai-microsoft-lawsuit.html} for more context on lawsuits filed against Stable Diffusion, GitHub, OpenAI respectively.}
These trends motivate the study of \emph{data marketplaces}, which aim to incentivize data sharing between sellers, that provide access to data, and buyers, that pay compensation for data access~\cite{castro2023data,agarwal2019marketplace,travizano2020wibson}.

In many applications, a data buyer has a specific goal in mind and, in particular, wants training data to predict their test data in a specific context.
Accessing different datapoints may require different prices associated with each datapoint, which may reflect heterogeneous cost, quality, or privacy levels for each datapoint~\cite{nielsen99whose,li2014theory}.
For example, the buyer may be a patient who cares about diagnosing their own chest X-ray and is willing to pay some fixed budget to access other X-ray images to train a model to have low error on their ``test set'' (see Figure~\ref{fig:teaser}).
However, not all X-ray images will be equally relevant. Thus, we want to select only those seller datapoints that are most useful for answering the buyer's query and fit the buyer's budget.

This goal of \emph{data acquisition} has motivated the development of many data-valuation techniques (e.g.~\cite[etc.]{ghorbani2019data,jia2019efficient,kwon2021beta,xu2021validation,sim2022data,kwon2023data,wang2023data, park2023trak,jiang2023opendataval}). However, we argue that current data valuation techniques are misaligned with the data acquisition problem, especially for running data markets. They all face at least one of the following hurdles:
\begin{itemize}
    \item The selection may not be adaptive to the buyer's (unlabeled) test queries and so might not select the most relevant data. In a data marketplace, we expect to buy only a small subset of the datapoints most relevant to the buyer's test data, which may have a very different distribution from the overall seller data distribution.
    \item If adaptive, the techniques rely on labeled validation data, which is often impractical. Further, when a small quantity of such data is available, the selection may overfit the validation data and result in poor performance on the test queries.
    \item The algorithms are not scalable and typically require retraining the ML model numerous times. Hence, they are unable to select from realistic seller corpora ($>$100K+ datapoints).
\end{itemize}

Instead, we propose a new Data Acquisition Via Experimental Design (DAVED) method that overcomes all of these limitations. 
% Further, it is able to be implemented in a distributed/federated manner, achieving lower prediction error even compared to centralized data valuation baselines.
Unlike most previous work in data valuation, our approach does not require a labeled validation dataset and instead directly optimizes data selection for the buyer's unlabeled test queries.
% This direct approach to data valuation avoids the issue of ``inference after selection,'' which can lead to overfitting on the validation set~\citep{vapnik1999overview,bartlett2002model}.
Additionally, our approach accounts for budget constraints and is able to weigh the price of each seller's data against the potential benefit, simultaneously solving the budget and revenue allocation problems~\cite{zhao2023addressing}. Further, it is amenable to a distributed/federated implementation and can thus power decentralized data markets.
In summary, our contributions are the following:
\begin{enumerate}%[noitemsep,topsep=0pt]
% \vspace{-0.1in}
    \item Formulate the data acquisition problem for data markets and show that data valuation methods make a fundamental theoretical mistake of ``inference after selection'' (Theorem~\ref{thm:data-shapley})
    \item Design DAVED (Data Acquisition Via Experimental Design), a novel, highly scalable, and distributed procedure that does not need validation data. Instead, it directly selects the most cost-effective seller data to answer the buyer's test queries.
     (Algorithm~\ref{alg:opt-procedue}),
    % \item Design a scalable iterative algorithm using the Frank-Wolfe method (\citeyear{frank1956algorithm}) to solve the objective.
    \item Demonstrate state-of-the-art performance on several synthetic benchmarks as well as on real-world medical datasets (MIMIC, Fitzpatrick17K, RSNA Bone Age, and DrugLib).
\end{enumerate}
% \vspace{-1em}

\begin{figure}
    \centering
    \includegraphics[width=0.95\textwidth]{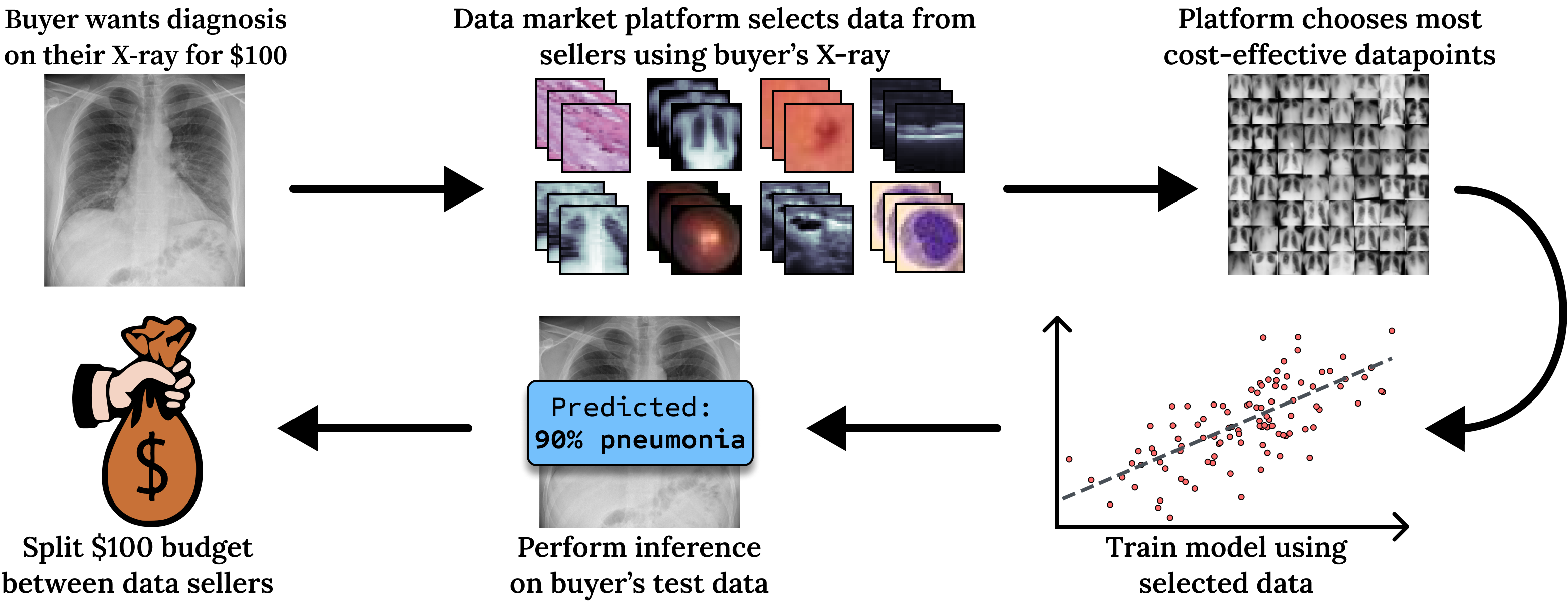}
    % \vspace{-0.5em}
    \caption{\textbf{Overview of data acquisition process between buyer and seller.} A buyer has a budget to acquire training data to get a prediction on their test query. The market platform optimizes the selection of seller data to be most useful for the buyer's query. The selected data is then used to train a regression model and make a prediction on the buyer's test data.
    % DAVED performs this optimization in a federated and scalable manner.
    }
    \label{fig:teaser}
    % \vspace*{-1em}
\end{figure}

\section{Data Acquisition versus Data Valuation}\label{sec:challenges}
% A major difference between our setting and previous work in data valuation is the lack of ``white-box'' access to a centralized data repository. 
In a decentralized data marketplace, data acquisition must be performed \emph{before} full data access is granted to the buyer~\cite{kennedy2022revisiting}.
This is related to Arrow's Information Paradox~\cite{arrow1972economic}, where a seller would be unwilling to give access before payment because data can be easily copied. 
At the same time, a buyer would be unlikely to purchase data without first determining its utility in the context of their data application. 
This distinction between data valuation and data acquisition for data markets is also discussed in a recent data acquisition benchmark, where data value must be estimated without requiring white-box access to the seller's data~\cite{chen2023data}.

A more fundamental issue with validation-based data valuation approaches is exemplified by the Data Shapley value approach~\cite[etc.]{ghorbani2019data,jia2019efficient,kwon2021beta}, which measures the marginal contribution of each training datapoint's improvement to a validation metric. 
They are great for after-the-fact attributing the relative influence of the training data.
However, they cannot be used to \emph{make decisions} about which datapoints should be included in the training.
This is because of, as noted earlier, the \emph{``inference after selection''} issue. Using validation data to select training data leads to substantial over-fitting to the validation data.

\begin{figure}
    \centering
    \includegraphics[width=1\textwidth]{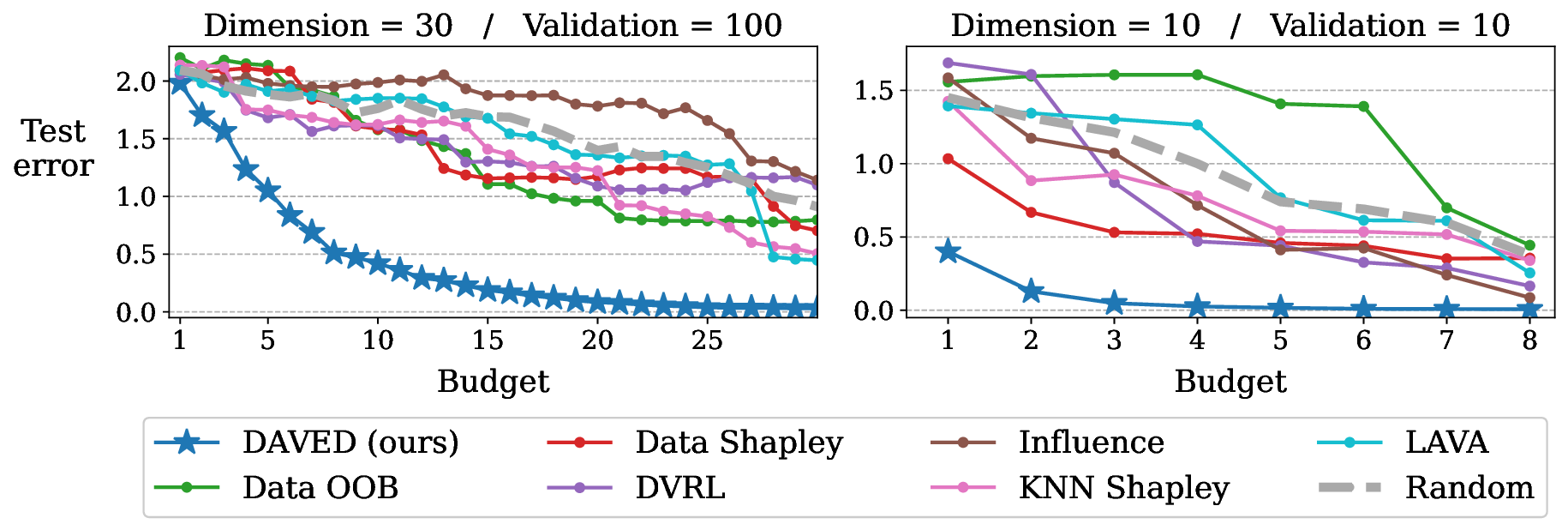}
    % \vspace{-1.0em}
    \caption{
    \textbf{Current data valuation methods overfit when data is high-dimensional or validation sets are small.}
    A total of 1,000 seller datapoints (each with cost 1) are available in the market. Each method selects training data for various budgets to train a regression model to predict the buyer's test data.
    The left plot shows that validation-based data valuation methods overfit when the data is high dimensional, while the right plot shows that they also overfit when the validation set is small. In contrast, our proposed DAVED achieves lower test error across a range of budgets of over 100 buyers. 
    }
    % \vspace{-0.5em}
    \label{fig:overfitting}
\end{figure}
\emph{Illustrative example:} Suppose that we only have a single validation datapoint. Then, it is clear that we will select training data similar to this singular datapoint, and our selection has no hope of working on the test dataset. While this clearly demonstrates overfitting in an extreme scenario, we show in Figure~\ref{fig:overfitting} that increasing the validation set size does not circumvent this issue.
% We distinguish between \emph{attributing} the influence of training data on validation performance and \emph{valuing} data for its usefulness in predicting unseen test data. 
% We contend that current data valuation methods, such as Data Shapley and its variants, that measure the change in the model's validation performance should not be conflated with the data value for training a model to achieve low test error.
% While these data attribution methods can help gather insight into model behavior on a particular datapoint, they can overfit the validation set, even on IID data.\footnote{Shapley-based techniques were initially used to estimate interpretability of each feature~\cite{vstrumbelj2014explaining}, such as the SHapley Additive exPlanations (SHAP) framework~\cite{lundberg2017unified}. Data Shapley~\cite{ghorbani2019data} repurposes Shapley to measure the influence of each datapoint rather than each feature; hence the name \emph{Data} Shapley.}
% \todo{sketch why}
% In ,  we illustrate this issue by showing how current data valuation methods overfit when the data dimension is large and when the validation set is small.
% For a simple experiment, the data seller has 1,000 available training datapoints, and the buyer has a single test datapoint, all i.i.d. 
% \todo{say how general this is}
% The goal is to select a subset of seller data to train a regression model to attain low prediction error on the buyer's data. 
% 
We see that other data valuation techniques have poor test prediction errors---some techniques even underperform even a random selection baseline! This clearly demonstrates overfitting.
Our proposed method maintains low test error as more seller training data is selected. In Section~\ref{sec:prior}, we will dig in deeper into this phenomenon and prove a very strong theoretical lower bound. We show that any approach that relies upon validation data for data selection can perform as badly as \emph{throwing away all the training data} and simply training on the validation data alone! This is especially true when our budget is small compared to the dimensionality of the problem, as is likely in a data market setting --- the data is typically high dimensional, and we can only select a very small fraction of the total available data.

\section{Setup and Limitations of Prior Methods}\label{sec:prior}
% \vspace*{-1em}
\textbf{Description of Data Acquisition Setting.}
As shown in Figure~\ref{fig:teaser}, in our setting of data acquisition, a buyer has a budget and test data. The platform uses the buyer's data to select training datapoints from the seller that optimize buyer test error.
The interaction proceeds as follows:
\begin{enumerate}
% [nosep]\vspace{-0.5em}
\item The buyer brings their test data $X^{\test} = \left[{x^{\test}_{1}, \dots, x^{\test}_{m}}\right]$  and a budget $B$ to the market platform. This test data is associated with unknown target labels $Y^{\test} = \left[{y^{\test}_{1}, \dots, y^{\test}_{m}}\right]$. 
\item The platform also has access to $n$ datapoints from data sellers $Z^{\train} = \{(x_j,y_j)\}_{j=1, \dots, n}$ and their associated costs $\{c_j\}_{j=1, \dots, n}$.
% \item The platform's task is to select the datapoints from sellers and train a model using this data. The trained model is then provided to the buyer, and the sellers receive payments.
\item Given \emph{only the covariates} $X^{\test}$, the platform assigns a selection weight $w_j$ to each datapoint $(x_j, y_j)$. This weight, $w_j \in \{0,1\}$ represents the discrete action of selecting datapoint $j$.
% the sampling probability $i$, (2) the ``usage'' of datapoint $i$ in training the model $f_\theta$, (3) or the privacy budget (e.g., the $\epsilon$ parameter in the differential privacy framework of~\citet{dwork2006calibrating}) associated with that particular datapoint in the algorithm.
\item The platform selects datapoints from the sellers according to $\mathbf{w} = (w_1,\dots, w_n)$, trains a model $f_{\hat \theta(\mathbf{w})}$, makes the predictions $f_{\hat \theta(\mathbf{w})}(X^{\test})$, and distributes $c_j \cdot w_j$ payment for each datapoint used in training.
% % \item Thus, the objective of the algorithm pick the $w$ which minimizes the expected prediction error for the buyer while adhering to their budget constraint i.e. $\sum_i w_i c_i \leq B$.
\end{enumerate}
In general, we do not make i.i.d assumptions between the train and test - we expect the test queries will not be similar to the total available train data. 
The goal then is to pick the weights $\mathbf{w}$, which minimizes the prediction error for the buyer while adhering to their budget constraint $\sumj w_j c_j \leq B$. 
% Note that the selection of the weights $w$ should depend on the platform's training algorithm which outputs $\hat \theta(w)$, estimate of the underlying parameter, using both training data $\left(X^{\train} = (x_1,\dots,x_n), Y^{\train} = (y_1,\dots,y_n)\right)$ and the weight $w$. 
This gives rise to the following problem:%\vspace*{-0.5em}
\begin{empheq}{flalign}\label{eq:platform_obj}
    \min_{\mathbf{w} \in \{0, 1\}^n} \cL(\mathbf{w}) &:= \frac{1}{m}\sumi \E\left[l\left(f_{\hat \theta(\mathbf{w})}(x^{\test}_i),y^{\test}_i\right)\right] \quad \text{s.t.} ~ \sumj w_j c_j \leq B. 
    % \min_{w \in \R^n_+} \cL(w) &:= \frac{1}{m}\sum_{i\in[m]} \E\left[l\left(f_{\hat \theta(w)}(x^{\test}_i),y^{\test}_i\right)\right] \quad \text{s.t.} ~ \sumj w_j c_j \leq B.\notag
\end{empheq}
Here, the expectation is over the conditional label distribution of $y^{\test}_i | x^{\test}_i$, and the potential randomness of the algorithm. Note that this problem can not be solved because we do not know the targets $Y^{\test}$. Instead, we need to rely on a \emph{proxy} $\hat\cL(\mathbf{w})$. One approach to constructing such a proxy is by using validation data.

% \subsection{Folly of Relying on Validation Data}\label{subsec:validation}
\textbf{Folly of Relying on Validation Data.}
In most data valuation methods, e.g. Data Shapley~\cite{ghorbani2019data}, the value of data is evaluated using a labeled validation set $Z^{\val} = \{(x_j^{\val},y_j^{\val})\}_{j=1}^{n_{\val}}$. Implicitly, these methods assume that the known $Z^{\val}$ is drawn from the same distribution as the unknown $Z^{\test}$. Then using this validation data, scores $(s_1, \dots, s_n)$ are assigned to the training datapoints $Z^{\train}$. For two datapoints $i,j \in \train$, the score $s_i > s_j$ if the datapoint $i$ is more valuable than $j$~\cite{park2023trak,jiang2023opendataval}. More concretely, $s_i > s_j$ implies that training with $i$ would lead to a smaller validation loss than if $j$ was used instead. Thus, these scores can used to select the most valuable datapoints.

However, note that we used the validation dataset to compute the scores. Thus, selecting the top-k scores results in implicitly minimizing the validation loss i.e., all validation-based data valuation schemes implicitly optimize the following proxy loss
\begin{align}\label{eq:data_shapley_opt}
    \min_{\mathbf{w} \in \{0,1\}^n} \hat\cL^{\val}(\mathbf{w}) := &~ \sum_{j=1}^{n_{\val}}l\left(f_{\hat \theta(w)}(x_j^{\val}), y_j^{\val}\right)  \quad \text{s.t.} ~ \sumj w_j c_j \leq B.
\end{align}
This approach heavily relies on the quantity and quality of the validation dataset in order to generalize to the actual test dataset. In fact, we have the following minimax lower bound even when restricting ourselves to simple linear models.
\begin{boxedtheorem}[Informal version of \Cref{thm:data-shapley-formal}]\label{thm:data-shapley}
    Let $\mathbf{w}^\ast$ denote the solution of our original problem~\eqref{eq:platform_obj} and $\mathbf{\hat w}$ solve~\eqref{eq:data_shapley_opt}. Suppose that all our data $Z^{\train}, Z^{\val}, Z^{\test}$ are drawn i.i.d. from some distribution $\D_{X,Y}$ where $\D_{X}$ is supported on $B_R^d$ (zero-centered ball with radius $R$ in $\R^d$), and $Y = \theta^\top X + \varepsilon$ where $\varepsilon$ is independent zero-meaned noise with variance $\sigma^2$. 
    For any training algorithm, when the number of training data is sufficiently large, with high probability,
\begin{align}
    \inf_{\hat \theta} \sup_{\D_{Y|X}} \E_{X^{\test}}\left[\cL(\mathbf{\hat w}) - \cL(\mathbf{w}^\ast)\right] \gtrsim \frac{\sigma^2d}{n_{\val}}. \notag
\end{align}
\end{boxedtheorem}

% \todo[]{Mention bound only applies to linear models (done)}

% \todo{mention what the $\otimes$ symbol is}
% \todo{what is $\lambda$?}
This result implies that the expected test error for any validation-based approach can in the worst-case scale as $d/n_{\val}$ with high probability.
% Similarly, its error on buyer's data will also scale as at least $d/n_{\val}$. 
% \todo[]{clarify what "constant probability" means}
This dependence on the dimension $d$ and the number of validation points $n_{\val}$ highlights that this method may be suboptimal in high-dimensional settings or when the validation dataset is small. In fact, we would get the same error scaling if we threw away the training data and trained a model $\hat\theta$ on the $n_{\val}$ validation datapoints alone. This explains the striking overfitting we observed earlier in Figure~\ref{fig:overfitting}. Furthermore, obtaining a large amount of ground-truth labeled data may be challenging in many real-world applications. Instead, we propose a validation-free approach to data acquisition based on experimental design. 
% conclude that such approaches, which includes the Data Shapley value approach of~\citet{ghorbani2019data}.

\section{Our Methods and Implementations} \label{sec:our-methods}
% \subsection{Linear experimental design}
% \vspace{-1em}

We take an alternative approach to define a proxy objective. First, we assume that the conditional distribution $\mathcal{D}_{y|x}$: $y = f_{\theta_*}(x) + \epsilon$ is identical across $Z^{\train}$ and $Z^{\test}$. If this does not hold, our problem becomes intractable since the same $x_j$ could have very different labels in the train and test set. We can recast our problem using the V-optimal experiment design framework~\cite{pukelsheim2006optimal}.

 % To address this problem, we model the targets $Y^{\test}$ and $Y^{\train} = (y_1,\dots,y_n)$ as having the , where $f$ is a parametric model, $\theta_*$ are the underlying parameters, and $e$ is the error term from a distribution $\mathcal{E}$ such that $\Var(\epsilon) = \sigma^2$. Furthermore, we let $\D_X$ denote the test distribution of the covariates $x^{\test}_i$. Note that we do not impose assumptions on $X^{\train}$, allowing distributions other than $\D_X$.  We contrast two approaches to solving this problem in the following subsections.

\textbf{Step 1: Linearizing the problem.} Our goal is to design a proxy loss function $\mathcal{\hat L}(\mathbf{w})$ which approximates the true test loss $\mathcal{L}(\mathbf{w})$. To do this, we have to reason about how different choices of training data $S \subset Z^{\train}$ could impact the prediction on a particular test datapoint in $X^{\test}$. This is a notoriously challenging problem for general deep learning models~\cite{basu2020influence}. Instead, we use a linear approximation and model the complicated training dynamics with kernelized linear regression. We suppose we have a known feature-extractor $\phi: \X \to \R^{d_0}$ and an unknown $\theta^\ast \in \R^{d_0}$ such that the data is generated as 
\begin{equation}\label{eq:linear-data}
    y = {\theta^\ast}^\top \phi(x) + \varepsilon\,,
\end{equation}
where $\varepsilon$ is independent noise with mean zero. The function $\phi(\cdot)$ can be the empirical Neural Tangent Kernel (eNTK) \cite{jacot2018neural,long2021properties,wei2021finetuned} of the model, or even the embeddings extracted from a deep neural network. While this may be a bad approximation in general~\cite{yang2020feature}, a recent line of work has shown that such eNTK representation very closely approximates the \emph{fine-tuning} dynamics of pre-trained models both theoretically~\cite{wei2021finetuned,malladi2023kernel} as well as emperically~\cite{fort2020deep,yu2022tct}. In fact, such linear approximations have also been used to speed up validation-based data attribution computations~\cite{park2023trak}.

\textbf{Step 2: Experimental design proxy.}
% \todo{Add citations to prior work in experiment design}
Given the assumption on our data from Eqn.~\eqref{eq:linear-data}, we can use the V-optimal experiment design framework~\cite{silvey1978optimal,pukelsheim2006optimal,huang2023evaluating} to define a proxy objective. First, suppose that $S \subseteq (\phi(X^{\train}), Y^{\train})$ is the subset selected by $\mathbf{w}$ and then we performed least-squares regression. The resulting estimate $\hat\theta(\mathbf{w})$ can be computed in closed form as
\[
\hat\theta(\mathbf{w}) = \big(\textstyle\sumj w_j \phi(x_j)\phi(x_j)^\top\big)^{\dagger} (\textstyle\sumj w_j \phi(x_j) y_j)\,.
\]
Henceforth, we will drop the $\phi$ when obvious from context and simply use $x$. We can further use Eqn.~\eqref{eq:linear-data} to compute the expected error on an arbitrary test query $x_0, y_0$ as follows:
\begin{align*}
% \vspace{-2em}
    \E[(\hat\theta(\mathbf{w})^\top x_0 - y)^2 | X^{\train}, x_0] &\stackrel{a_1}{=} \E[\big((\hat\theta(\mathbf{w}) - \theta^\ast)^\top x_0 + \varepsilon \big)^2 ]\\
    &\stackrel{a_2}{=} x_0^\top \E[(\hat\theta(\mathbf{w}) - \theta^\ast)(\hat\theta(\mathbf{w}) - \theta^\ast)^\top] x_0 + \E\|\varepsilon\|^2\\
    &\stackrel{a_3}{=} x_0^\top \E[\hat\theta(\mathbf{w})\hat\theta(\mathbf{w})^\top] x_0 + \E\|\varepsilon\|^2\\
    &\stackrel{a_4}{=} x_0^\top \big(\underbrace{\textstyle\sumj w_j x_jx_j^\top}_{=: \mathcal{I}(\mathbf{w})}\big)^{\dagger} x_0 + \E\|\varepsilon\|^2
\vspace{-2em}
\end{align*}
Here $a_3$ uses the unbiasedness of the OLS estimator and $a_4$ plugs in the closed form of $\hat\theta(\mathbf{w})$ and simplifies. 
% Here, $a_1$ and $a_2$ simply use~\eqref{eq:linear-data}, $a_3$ uses the fact that the OLS estimator $\hat\theta(\mathbf{w})$ is unbiased, and finally $a_4$ plugs in the explicit form of $\hat\theta(\mathbf{w})$ and simplifies. 
With this, we end up with a very clean expression for the expected test error on an arbitrary point $x_0$, and the matrix $\mathcal{I}(\mathbf{w})$ is known as the \emph{Fisher information matrix}. While regression suffices for our use case, the procedure can be extended to general linear models. Dropping the fixed $\E\|\varepsilon\|^2$, we can use this to build our proxy function $\mathcal{\hat L}^{ED}(\mathbf{w})$ and
% \begin{align}
%     \mathcal{\hat L}^{ED}(\mathbf{w}) := \frac{1}{m}\sumi\E\left[l\left(f_{\hat \theta(\mathbf{w})}(x^{\test}_i),y^{\test}_i\right) \, | \, X^{\train}, X^{\test}\right]
%     = \frac{1}{m}\sum_{i\in[m]} \left(x^{\test}_i\right)^\top \mathcal{I}(\mathbf{w})^{\dagger} x^{\test}_i.
% \end{align}
arrive at the following optimization problem
\begin{empheq}[box=\fbox]{align}\label{eq:design_opt}
    \min_{\mathbf{w} \in \{0, 1\}^n} &~ \left\{\mathcal{\hat L}^{ED}(\mathbf{w}) := \nicefrac{1}{m}\textstyle\sumi (x^{\test}_i)^\top \mathcal{I}(\mathbf{w})^{\dagger} (x^{\test}_i) \right\} \quad \text{s.t.} ~ \textstyle\sumj w_j c_j \leq B.
\end{empheq}
Note that our proxy function $\mathcal{\hat L}^{ED}(\mathbf{w})$ can be computed using just $X^{\train}, X^{\test}$ and does not even need access to training labels. Unfortunately, the objective in \eqref{eq:design_opt} is NP-hard to optimize~\cite{allen2021near}. We next see how to derive fast and provably good approximation algorithms for \eqref{eq:design_opt}.

\textbf{Step 3: Fast approximation.}
To make Eq.~\ref{eq:design_opt} amendable to gradient-based optimization, we drop the constraint that $w_j \in \{0,1\}$ and allow it to be a continuous positive vector i.e. $\mathbf{w} \geq 0\text{ and } \sumj w_jc_j \leq B$. With this relaxation, the proxy objective $\mathcal{\hat L}^{ED}(\mathbf{w})$ is continuous and convex in $\mathbf{w}$~\cite{boyd2004convex}. We then run the ``herding'' variant of the \emph{Frank-Wolfe} algorithm~\cite{wynn1970sequential,khachiyan1996rounding,sun2004computation,bach2012herding,zhao2023analysis}. To do this, define $\left( \mathbf{\tilde w}_t := \mathbf{w}_t / \mathbf{c} \right)$ for any $t$. We start from a $\mathbf{\tilde w_0} = \mathbf{e}_0$ and iteratively update as\footnote{Here, $\mathbf{e}_j$ is the standard basis vector along axis $j$, and in $\mathbf{\tilde w} := \mathbf{w} / \mathbf{c}$, the division is performed element-wise.}
\begin{equation}\label{eq:FW-herding-update}
  \mathbf{\tilde w}_{t+1} \leftarrow (1 - \alpha_t)\mathbf{\tilde w}_t +   \alpha_t\mathbf{e}_{j_t}, \quad \text{ where } j_t = \arg\max_{j \in [n]} (-\nabla_{w_j} \mathcal{\hat L}(\mathbf{w}_t)/{c_j}) 
\end{equation}
Note that if we use the step-size $\alpha_t = \tfrac{1}{t+1}$ in \eqref{eq:FW-herding-update}, $\mathbf{w}_t$ satisfies a special property at any iteration $t$:
\[
     \mathbf{\tilde w}_t \in \Delta^n \text{ and further } (t+1)\mathbf{\tilde w}_t  \in \{0,1 \}^n\,.
\]
Run the procedure until the last iteration $t = t_o$ for which we still have $\|\mathbf{w}_{t_0}\|_1 \leq B$.
We can adapt the theory from \cite{bach2012herding,jaggi2013revisiting} to analyze the above procedure and show the following.
\begin{boxedtheorem}[Informal]\label{thm:herding}
    Let us run Frank-Wolfe herding update \eqref{eq:FW-herding-update} for $t_0$ steps such that it is last step which satisfies $\|\mathbf{w}_{t_0}\|_1 \leq B$. We use $\mathbf{\tilde w}_{t_0} = ((t_0+1)\mathbf{w}_{t_0} / \mathbf{c})$ as our selection vector and we would have selected $t_0$ datapoints. Then, under some assumptions, we have
\begin{align*}
    \mathcal{\hat L}^{ED}\big((t+1)\mathbf{\tilde w}_{t_0}\big) \leq \min_{\mathbf{w} \in \{0,1\}^n, \sumj w_jc_j \leq B}\mathcal{\hat L}^{ED}(\mathbf{w}) + O\bigg(\frac{\log t_0}{t_0}\bigg)\,.
\end{align*}
\end{boxedtheorem}
The above theorem shows that our continuous relaxation does not significantly affect the optimality of our result--we get $O(\frac{\log t_0}{t_0})$ close to the optimal solution to the original NP-hard \eqref{eq:design_opt}. If all datapoints have equal cost $c$, then $t_0 = \lfloor B/c \rfloor$, and so our approximation quality improves as we increase the budget. While better approximation guarantees are attainable~\cite{allen2019convergence}, their procedure is significantly more involved and is not easily amenable to efficient federated implementations as ours is.
% Further, as noted in~\cite{bach2012herding}, one can show stronger guarantees with additional assumptions for our method as well.
% If we further constrain $w$ to be in the probability simplex $\Delta^n$.
% then our cost objective is the following:
% \setlength{\abovedisplayskip}{2pt}
% \setlength{\belowdisplayskip}{2pt}
% \begin{gather}\label{eq:cost-function}
%     \underset{\mathbf{w} \in \Delta^n}{\arg \min} \left\{ C(\mathbf{w}) \triangleq \underset{z^\test \sim \D}{\E} \left[\left(z^\test\right)^\top \, P(\mathbf{w}) \, z^\test\right] \right\}, \quad \text{s.t.} ~ \sumj w_j c_j \leq B.
% \end{gather}
% where $P \in \R^{d \times d}$ is the inverse information matrix, defined as 
% \begin{gather}\label{eq:inverse-info-mat}
% P(\mathbf{w}) \triangleq \bigg(\sum_{j=1}^{n}w_j x_jx_j^\top\bigg)^{-1}\,.
% \end{gather}
% $\arg\max_{i \in [n]} |\nabla_{w_i} \mathcal{\hat L}(\mathbf{w}_t)/{c_i}|$

\textbf{Step 4: Efficient federated implementation.} Our practical implementation directly restricts $\mathbf{w} \in \Delta^n$ instead of $\mathbf{\tilde w}$ in the theoretical implementation above i.e. we run
% \vspace*{-0.5em}
\begin{equation}\label{eq:practical-update}
  \mathbf{w}_{t+1} \leftarrow (1 - \alpha_t)\mathbf{w}_t +   \alpha_t\mathbf{e}_{j_t}, \quad \text{ where } j_t = \arg\max_{j \in [n]} (-\nabla_{w_j} \mathcal{\hat L}(\mathbf{w}_t)/{c_j}) 
\end{equation}
This way $\mathbf{w}$ can be directly interpreted to be the sampling probability for different training datapoints. The bottleneck to efficiently implementing \eqref{eq:practical-update} is computing the gradient. At step $t$, the negative gradient can be shown to be
% \vspace*{-1em}
\begin{align}\label{eq:cost_grad}
 g_j := -\nabla_{w_j} \mathcal{\hat L}(\mathbf{w}_t) = \nicefrac{1}{m}\textstyle\sumi \left((x^{\test}_i)^\top \mathcal{I}(\mathbf{w}_t)^{\dagger} (x^{\train}_j)\right)^2.
\end{align}
Thus, if we have the inverse information matrix $\mathcal{I}(\mathbf{w}_t)^\dagger$ pre-computed, $g_j$ as well as the update \eqref{eq:practical-update} can be trivially computed by seller $j$ using only their data $x^{\train}_j$ (and the test data). Next, we show how to efficiently maintain the inverse information matrix.
Note that the FW update~\eqref{eq:practical-update} has a special structure: all coordinates are shrunk, and then only a single coordinate of $\mathbf{w}_t$ is increased. We can relate the resulting $\mathcal{I}$ matrices with a rank-one update as:\
\[
\mathcal{I}(\mathbf{w}_{t+1}) = (1-\alpha_t)\mathcal{I}(\mathbf{w}_{t}) + \alpha_t x_{j_t}x_{j_t}^\top\,.
\]
Define $P_t := \mathcal{I}(\mathbf{w}_{t})^\dagger$. We can use the Sherman–Morrison formula~\citep{sherman1950adjustment} to compute $P_{t+1} = \mathcal{I}(\mathbf{w}_{t+1})^\dagger$ as 
% \vspace*{-1.5em}
% ancan be performed in a si extremely straightforward 
% matriGiven the value of $P(\mathbf{w})$, the above gradient is easy to compute using only local seller data $x_j$. Hence, the gradient update can be federated, where only the coordinates corresponding to the datapoint $x_j$ needs to be communicated.\\
% Our optimization strategy uses the Frank-Wolfe algorithm~\citep{frank1956algorithm,jaggi2013revisiting} to iteratively update the weight vector $\mathbf{w}$. At each iteration, we shrink all weights by $(1-\alpha)$ and increase the weight of the coordinate with the largest negative gradient by the step size $\alpha$ moderated by the cost of that datapoint, $c_j$,
% \begin{align}\label{eq:weight-update}
%  \mathbf{w}_{t+1} = (1-\alpha)\mathbf{w}_t +\alpha \mathbf{e}_{j_t}, \quad \text{where}~~
%  j_t  = \underset{j}{\arg \max} \left(\frac{-\nabla_{w_j} C(\mathbf{w})}{c_j}\right).
% \end{align}
% Computing the max gradient coordinate requires communication among the sellers, but is very communication-efficient to implement --- they only send a single scalar.
% \todo[]{Cite Experimental Design on a Budget for Sparse Linear Models and Applications}
% Further, note that with Frank-Wolfe, the update to the information matrix will be \emph{rank 1} at each iteration. Hence, we can efficiently update the $P(\mathbf{w})$ matrix using the Sherman–Morrison formula~\citep{sherman1950adjustment}:
\begin{align}\label{eq:sherman-morrison}
    P_{t+1} = \frac{1}{1-\alpha_t}P_t - \frac{\alpha_t P_t x_{j_t} x^\top_{j_t} P_t}{1 - \alpha_t + \alpha_t x_{j_t}^\top P_t x_{j_t}}.
\end{align}
For each round $t$, this update only involves the current matrix $P_t = \mathcal{I}(\mathbf{w}_{t})^\dagger$ and the single datapoint $x_{j_t}$ selected for the round. Thus, the seller can also locally compute this update as well as the updated cost $\mathcal{\hat L}(\mathbf{w})$ as in Eq.~\eqref{eq:design_opt} for any $\alpha_t$:  
% \vspace*{-.5em}
\begin{align}\label{eq:optimal-line-search}
      \mathcal{\hat L}(\mathbf{w}) = \tfrac{1}{m(1-\alpha_t)} \textstyle\sumi \big({x_i}^\top \, P_t \, x_i \big) - \frac{\alpha_t}{1 + \alpha_t x_{j_t}^\top P_t x_{j_t}} \textstyle\sumi\big(  {x_i}^\top P_t \, x_{j_t}\big)^2
\end{align}
Thus, a line search can be performed to determine the optimal step size $\alpha_t \in [0,1]$ to minimize the proxy loss as. This differs from ~\eqref{eq:FW-herding-update} where we used a specific choice of $\alpha_t$. FW is known to be more stable with the line search~\cite{sun2004computation,zhao2023analysis}. The seller can communicate this $\alpha_t$ and $x_{j_t}$ to the platform to compute the updated $P_{t+1}$ using only $O(d)$ communication.
% \textbf{Regularization of the Initialization of Information.}
An additional practical consideration is that by initializing $\mathbf{w_0} = c_1 \mathbf{e}_1$, we have an ill-conditioned inverse information matrix $P_0$. We instead use an initialization of $\mathbf{w_0} = \mathbbm{1}_n / n \in \Delta^n$ and further add a small regularization term. This makes the initial $P_0$
% \vspace{-1.5em}
\begin{gather}\label{eq:reg-init-matrix}
    P_0 = \left(\left(1 - \lambda_\text{Reg}\right) X^\top \; \text{diag}(\mathbf{w_0}) X \;   + \lambda_\text{Reg} \cdot \sigma_X^2 I_{n\times n} \right)^{-1},
\end{gather}
where $I_{n\times n}$ is the identity matrix and $\sigma_X^2$ is the variance of the training data $X$. 
We typically use $\lambda = 1$ since, in this case, $P_0$ does not depend on the seller data $X$ at all. However, we find that a smaller $\lambda$ could improve performance.
The complete details are summarized in Algorithm~\ref{alg:opt-procedue}.

\textbf{Single-step variant.} We can also forgo the iterative process and instead linearly approximate the cost function (Eq~\ref{eq:design_opt}) with a \emph{single step} that selects the top $k$ datapoints under the budget $B$,
% \vspace{-0.5em}
\begin{empheq}[box=\fbox]{align}\label{eq:single-step}
    \mathtt{single\_step}(x^\test, X, B) = \mathtt{top\_k} \big(\big\{\textstyle\sumi \big[ (x_i^\test)^\top P_0 x_j \big]^2\big\}_{j = 1}^n\big).
    % \vspace{-0.5em}
\end{empheq}
This simplified version is extremely fast while still maintaining relatively good performance. 
% \vspace{-2em}

% \setlength{\floatsep}{2pt}
% \setlength{\textfloatsep}{2pt}
% \setlength{\intextsep}{2pt}
\begin{algorithm}[t]\label{alg:opt-procedue}
% \small
\caption{DAVED: Iterative Optimization Procedure}
\begin{algorithmic}[1]
     \renewcommand{\algorithmiccomment}[1]{\hfill\# #1}
    \STATE \textbf{Input:} buyer test datapoint $X^\test \in \mathbb{R}^{m \times d}$, seller training data $X \in \mathbb{R}^{n \times d}$, seller weights $\mathbf{w} \in \Delta^n$, iteration steps $T$, regularization parameter $\lambda_\text{Reg} \in [0, 1]$, and seller datapoint costs $\mathbf{c} \in \mathbb{R}_+^n$
    \STATE $\mathbf{w}_0 \leftarrow \mathbbm{1} / n$ \COMMENT{Initialize seller weight vector to uniform distribution}
    \STATE $P_0 \leftarrow \left(\left(1 - \lambda_\text{Reg}\right) X^\top \; \text{diag}(\mathbf{w}_0) \; X  + \lambda_\text{Reg} \cdot \sigma_X^2 I_{n\times n} \right)^{-1}$ \COMMENT{Initialize inverse information matrix (Eq.~\ref{eq:reg-init-matrix})}
    \FOR{ $t \in \{1, 2, \ldots, T\}$}
    % \STATE $\mathcal{L}_t \leftarrow \mathbb{E}_{z^\test \sim \D}[\left(x^\test\right)^\top P(\mathbf{w}_t) x^\test]$ \COMMENT{Compute design objective function (Eq.~\ref{eq:design_opt})}
    \STATE $g \leftarrow -\nabla \mathcal{\hat L}(\mathbf{w}_t)$ \COMMENT{Compute negative gradients (Eq.~\ref{eq:cost_grad}})
    \STATE $j_t \leftarrow {\arg\max}_j \; \left(\nicefrac{g_j}{c_j}\right)$ \COMMENT{Select coordinate based on costs}
    \STATE $\alpha_t \leftarrow \mathtt{LINE\_SEARCH}(\mathcal{\hat L})$ \COMMENT{Find optimal step size (Eq.~\ref{eq:optimal-line-search})}
    \STATE $\mathbf{w}_{t+1} \leftarrow (1-\alpha_t)\mathbf{w}_t +\alpha_t \mathbf{e}_{j_t}$ \COMMENT{Shrink weights and upweight the chosen coordinate (Eq.~\ref{eq:practical-update})}
    % \STATE $w_{j_t} \leftarrow w_{j_t} + \alpha$ \COMMENT{Increase weight}
    % \STATE $x_{j_t} \leftarrow [X]_{j_t}$ \COMMENT{Select seller datapoint of chosen coordinate}
    \STATE $P_{t+1} \leftarrow \mathtt{SHERMAN\_MORRISON}(P_t, x_{j_t}, \alpha_t)$ \COMMENT{Update inverse information matrix (Eq.~\ref{eq:sherman-morrison})}
    \ENDFOR
    \STATE \textbf{Output:} Sample seller data according to $\mathbf{w}_T \in \Delta^n$ without replacement until budget $B$ runs out.
\end{algorithmic}
\end{algorithm}
% \vspace{2em}
% \setlength\abovecaptionskip{-4pt}
% \setlength\belowcaptionskip{4pt}
\begin{figure}
    \centering
    \includegraphics[width=1\textwidth]{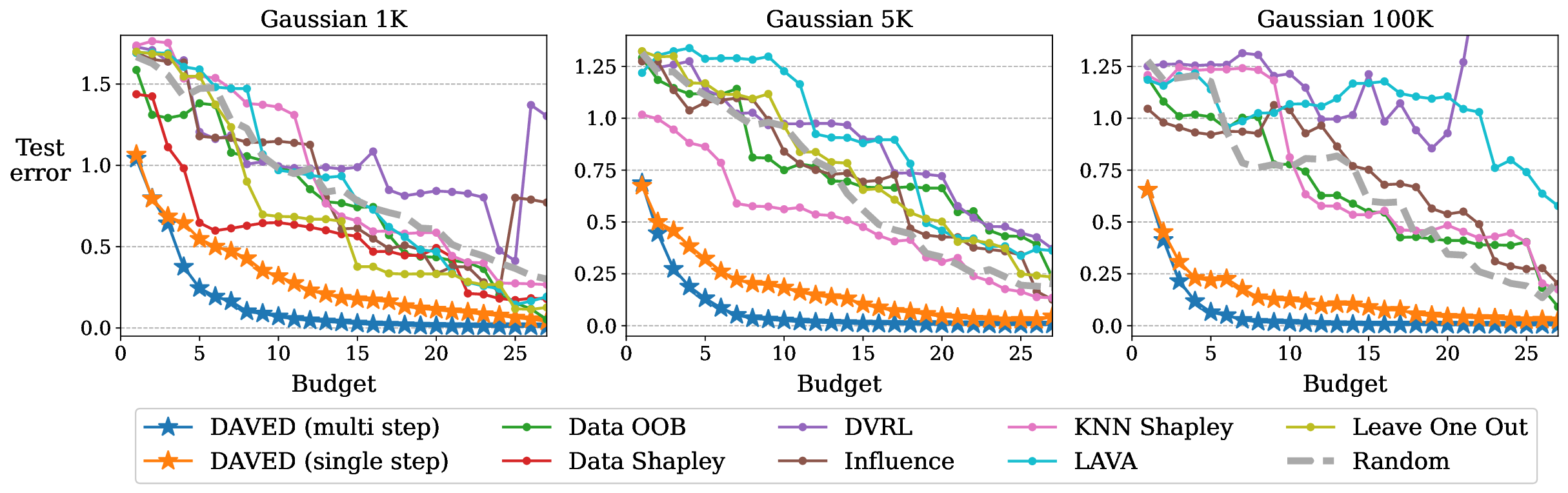}
    \caption{
    \textbf{DAVED achieves better test error against other methods on synthetic data.}
    For three amounts of total seller data (1K, 5K, 100K), each method is optimized to select the most valuable training datapoints from the seller to predict the buyer's test data.
    Our data selection algorithm based on experimental design achieves better test MSE on the buyer's data with fewer training points.
    % Results are plotted as the average over 100 different buyers. 
    }
    \label{fig:gauss-num-seller}
\end{figure}

\begin{figure}
    \centering
    \includegraphics[width=1\textwidth]{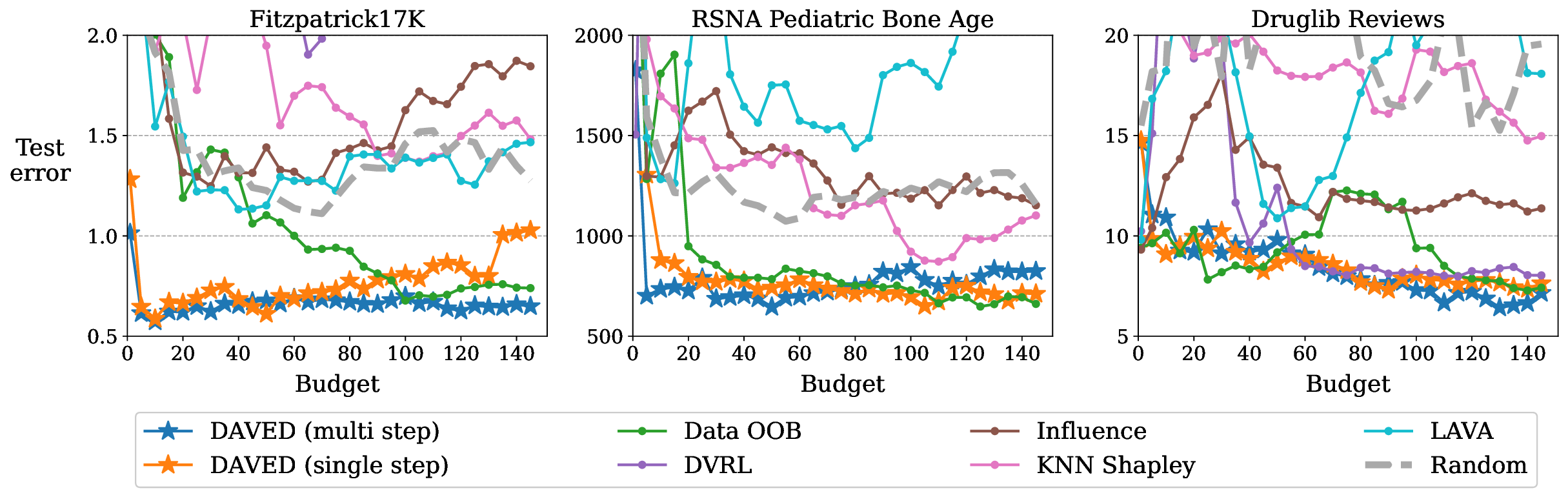}
    \caption{
    \textbf{DAVED has low test error on real-world medical imaging and drug review data.}
    We compare our method against other data valuation methods on two medical imaging datasets (Fitzpatrick17K and RSNA Bone Age) embedded through CLIP and the DrugLib review dataset embedded through GPT-2. 
    After optimization, each data valuation method selects the top-$k$ most valuable datapoints from the seller to train a regression model to predict the buyer's test data.
    Our data selection algorithm based on experimental design achieves lower prediction mean squared error on the buyer's data with fewer training points.
    % Results are plotted as the average over 100 different buyers. 
    }
    \label{fig:embedding}
\end{figure}

\begin{table}[ht]
    \centering
    % \small
    \begin{tabular}{l|c*{6}{|c}}
        \toprule
        \multirow{2}{*}{\textbf{Method}} & \multicolumn{2}{c|}{\textbf{Gaussian}} &  \multicolumn{2}{c|}{\textbf{MIMIC}} &  \multicolumn{1}{c|}{\textbf{RSNA}} & \multicolumn{1}{c|}{\textbf{Fitzpatrick}} & \multicolumn{1}{c}{\textbf{DrugLib}} \\
        % \cline{2-6}
         & \textsc{1K} & \textsc{100K} & \textsc{1K} & \textsc{35K} & \textsc{12K} & \textsc{15K} & \textsc{3.5K}\\
        \hline
        Random baseline                        & 1.38 & 1.01 & 301.0 & 283.7 & 1309.1 & 1.49 & 21.4 \\
        Data Shapley~\cite{ghorbani2019data}   & 0.87 & N/A  & 294.9 & N/A   & N/A    & N/A  & N/A \\
        Leave One Out~\cite{cook1977detection} & 1.31 & N/A  &1125.0 & N/A   & N/A    & N/A  & N/A \\
        Influence~\cite{feldman2020neural}     & 1.47 & 0.97 & 189.4 & 876.4 & 1614.5 & 1.93 & 12.8 \\
        DVRL~\cite{yoon2020data}               & 1.33 & 1.26 & 229.7 & 285.5 & 3528.8 & 3.00 & 12.6 \\
        LAVA~\cite{just2023lava}               & 1.47 & 1.10 & 190.9 & 417.3 & 1867.5 & 1.45 & 17.4 \\
        KNN Shapley~\cite{jia2019efficient}    & 1.55 & 1.18 & 175.7 & 229.6 & 1387.0 & 1.82 & 19.0 \\
        Data OOB~\cite{kwon2023data}           & 1.24 & 0.98 & \textbf{169.7} & \underline{215.6} & 1020.3 & 1.35 & 10.0 \\
        DAVED (single step)                     & \underline{0.58} & \underline{0.27} & 277.4 & 659.9 &  900.2 & \underline{0.73} &  \textbf{9.0} \\
        DAVED (multi-step)                      & \textbf{0.37} & \textbf{0.16} & \underline{206.7} & \textbf{171.4} &  \textbf{785.2} & \textbf{0.67} &  \underline{9.2} \\
        \bottomrule
    \end{tabular}
    \caption{\textbf{Our method achieves a lower buyer error than other common data valuation methods.} We compared the test mean squared error on the buyer test point on a synthetic Gaussian dataset and four medical datasets: MIMIC, RSNA, Fitzpatrick17K, and DrugLib. The subheading denotes the number of seller training data available for that experiment, and ``N/A'' denotes that the method exceeded runtime constraints for the experiment. 
    We optimize a separate random sample of training and validation data for each buyer and average over 100 buyers. 
    Bolded values indicate the best-performing method and underlined values denote the second-best-performing method.
    }
    \label{tab:error-results}
\end{table}

\begin{table}
    \centering
    \small
    \begin{tabular}{l*{10}{|c}}
        \toprule
            & \multicolumn{2}{c|}{\textbf{Gaussian}} &  \multicolumn{2}{c|}{\textbf{MIMIC}} &  \multicolumn{2}{c|}{\textbf{RSNA}} & \multicolumn{2}{c|}{\textbf{Fitzpatrick}} & \multicolumn{2}{c}{\textbf{DrugLib}} \\
        % \cline{2-6}
        \midrule
        \textsc{Cost function} & $\sqrt{c}$ & $c^2$ & $\sqrt{c}$ & $c^2$ & $\sqrt{c}$ & $c^2$ & $\sqrt{c}$ & $c^2$ & $\sqrt{c}$ & $c^2$ \\
        \midrule
        % Data Shapley        &   &   &   &   &   &   &   &   &  \\
        % Leave One Out       &   &   &   &   &   &   &   &   &  \\
        % Influence subsample &   &   &   &   &   &   &   &   &  \\
        Random baseline     & 2.26 & 77.7 & 288.3 & 284.6  &  2253.8 & 2065.0 & 2.14  & 2.11 & 16.5 & 18.4 \\
        DVRL                & 1.67 &  \underline{2.1} & \underline{213.9} & \underline{214.5}  & 24003.1 & 5588.3 & 1.46  & 8.90 & 22.5 & 20.7 \\
        LAVA                & 2.13 &  3.3 & 482.1 & 474.6  &  \underline{1666.9} & 1586.7 & 2.09  & 2.21 & 35.7 & 34.8 \\
        KNN Shapley         & 2.13 & 69.0 & 217.1 & 956.0  &  2753.9 & 2505.8 & 1.85  & 2.15 & 13.6 & 13.0 \\
        Data OOB            & 2.19 &  3.3 & 242.8 & 245.4  &  1695.3 & \underline{1205.2} & 2.08  & 2.52 & \underline{10.8} & \underline{10.8} \\
        DAVED (single step) & \underline{1.54} & 251.5 & 598.4 & 585.1  & 1734.3 & 1550.0 & \textbf{0.75}  & \textbf{0.71} &  \textbf{9.4} & \textbf{10.1} \\
        DAVED (multi-step)  & \textbf{0.04} &  \textbf{0.2} &  \textbf{168.9} & \textbf{168.7}  & \textbf{1076.3} &  \textbf{942.1} & \underline{0.76}  & \underline{0.75} & 12.6 & 11.4 \\
        \bottomrule
    \end{tabular}
    \caption{Comparing data selection methods for two different cost functions, $\sqrt{c}$ and $c^2$. For each budget constraint, we select seller points until the budget is exceeded and calculate test prediction error on the buyer data. We average over 100 buyers and report the mean test error across budgets from 1 to 30.}
    \label{tab:budget-results}
\end{table}

\section{Experiments}\label{sec:exp}
% \vspace{-1em}
We evaluate our proposed method for data acquisition (DAVED) against common data valuation methods on synthetic Gaussian-distributed data as well as four real-world medical datasets:
\begin{enumerate}
% [nosep]
    \item \textbf{Fitzpatrick17K}~\cite{groh2021evaluating}, a skin lesion dataset, where the task is to predict Fitzpatrick skin tone on a 6-point scale from dermatology images.
    \item \textbf{RSNA Pediatric Bone Age dataset}~\cite{halabi2019rsna}, where the task is to assess bone age (in months) from X-ray images of an infant's hand.
    \item \textbf{Medical Information Mart for Intensive Care (MIMIC-III)}~\cite{johnson2016mimic}, where the task is to predict the length of hospital stay from 48 attributes such as demographics, insurance, and medical conditions.
    \item \textbf{DrugLib reviews}~\cite{misc_drug_review_dataset_(druglib.com)_461}, text reviews of drugs where the task is to predict ratings (1-10).
\end{enumerate}
    For Fitzpatrick17K and RSNA Bone Age datasets, each image was embedded through a CLIP ViT-B/32 model~\cite{radford2021learning}, while for the DrugLib dataset, each text review was embedded through GPT-2 model ~\cite{radford2019language} with a max context length of 4096. 
    For validation-based methods, we use a validation set of 100 datapoints. We report mean test errors over 100 buyers.
    For more details on the experimental setup, see Appendix~\ref{app:setup}.
    % In addition to synthetic Gaussian data (see Appendix~\ref{app:setup} for code implementation), we also use the following medical datasets for evaluation: \vspace{-0.5em}

\textbf{Comparing Performance on Data with Homogeneous Costs.}
    In Figure~\ref{fig:gauss-num-seller}, we evaluate our method and several other data valuation methods on varying amounts of Gaussian data with homogeneous fixed costs.
    Compared to other methods, both multi- and single-step versions of DAVED have lower test errors across budgets on synthetic data. 
    This performance gap is especially large when the buyer has a small budget (around 5-10 training datapoints). 
    In Figure~\ref{fig:embedding}, we evaluate our method on real image and text data embedded through CLIP and GPT-2 feature representations. We observe that DAVED has better performance compared to most other baselines on all three datasets, highlighting that the proposed method is practical for embeddings of high-dimensional data.
    Table~\ref{tab:error-results} summarizes our results on all datasets.
    For the Gaussian and MIMIC datasets, we report the mean error of budgets from 1 to 10, while for the embedded datasets (RSNA, Fitzpatrick17K, DrugLib), we report the mean error of budgets from 1 to 100 in intervals of five.

    % such as when $x$ is embedded through deep representations, when $\{x_{j}\}_{j=1}^{n}$ has heterogeneous costs $\{c_{j}\}_{j=1}^{n}$, and when the buyer data $x_0$ consists of multiple test data points $\{x_i\}_{i=1}^{m}$.
    % We compare scalability between selection algorithms and also perform several ablation experiments to evaluate the choice of optimization steps and the amount of regularization on the initialization of the information matrix.

\textbf{Comparing Performance on Data with Heterogeneous Costs.} 
    Next, we compare methods on seller data with non-homogeneous costs.
    We uniformly sample costs $c \in \{1, 2, 3, 4, 5\}$ for each seller point and consider two cost functions, $c_j = \sqrt{c}$ and $c_j = c^2$,  which downweights gradient of that datapoint $x_j$ (see Equation~\ref{eq:cost_grad}).
    To simulate heterogeneous utility across datapoints, we introduce cost-dependent label noise, $\epsilon \sim \mathcal{N}(\bar{y}, \sigma^2)$, to each datapoint $\tilde{y_i} := y_i + \beta \tilde{\epsilon} / c_j$, where $\bar{y}$ is the mean target value and $\beta$ is the overall noise level, which we fix at $30\%$ throughout our experiments.
    For these experiments, we did not evaluate Data Shapley~\cite{ghorbani2019data}, LOO~\cite{cook1977detection}, and Influence~\cite{feldman2020neural} that had very long runtimes.
    In Table~\ref{tab:budget-results}, we report additional mean test error across budgets 1--30 for both cost functions. 
    We find that our DAVID method is more budget-efficient in choosing cost-effective noisy datapoints than other methods across datasets. 
    We provide additional plots for heterogeneous costs in Appendix~\ref{app:budget}.
    
\textbf{Comparing Runtime.}
    In Figure~\ref{fig:runtime}, we compare the optimization runtime of our data selection method on 1,000 datapoints while increasing the dimensionality of the data as well as when the dimensionality is fixed to 30, and the number of seller datapoints is increased to 100,000. 
    Data Shapley~\cite{ghorbani2019data} and LOO~\cite{cook1977detection} took too long to run for large amounts of datapoints or high dimensional data and are not reported. 
    In both experiments, our multi-step compares favorably to efficiency-optimized techniques such as KNN Shapley~\cite{jia2019efficient} while our single-step method had the fastest runtime. This demonstrates that our method can scale to marketplaces with millions of datapoints. 

\begin{figure}
    \centering
    \includegraphics[width=0.99\columnwidth]{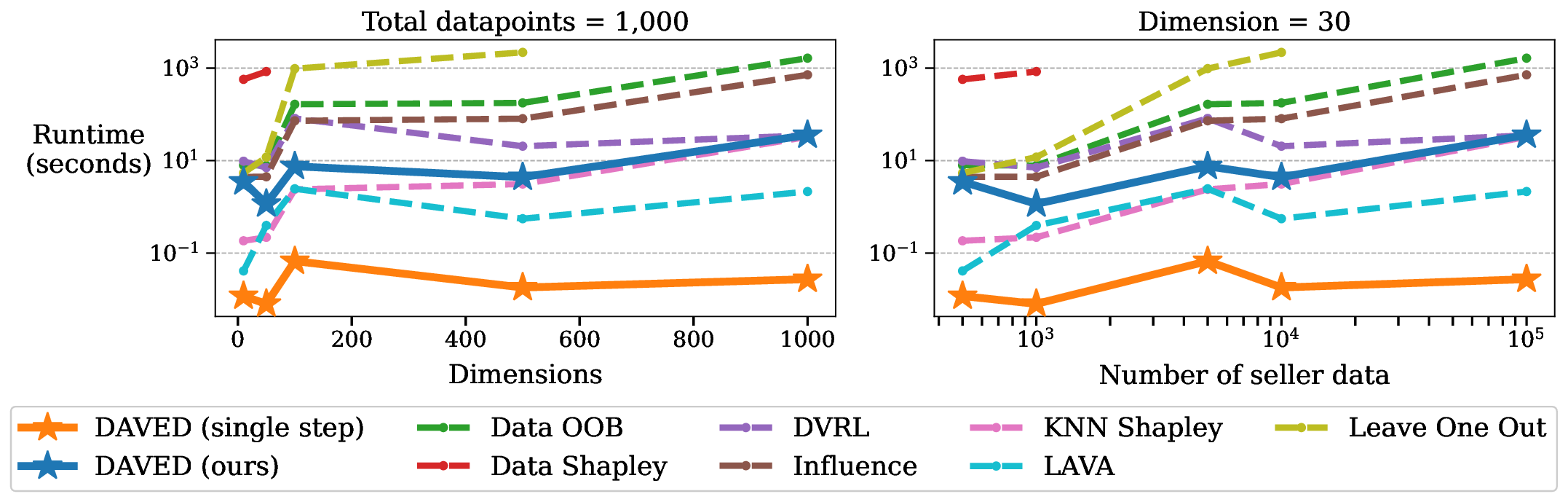}
    \caption{\textbf{DAVED has lower runtime than model-based data valuation methods.}
        The left subplot shows runtimes of varying data dimensions when fixing the number of datapoints at 1,000. while the right subplot shows the runtimes of varying the number of total seller datapoints when the dimension is fixed. In both cases, we see that our method is orders of magnitude faster than Data Shapley and the single-step variant of our method is faster than even optimized data valuation methods such as KNN Shapley and model-free methods such as LAVA.
    }
    \label{fig:runtime}
\end{figure}

\textbf{Regularization Strength.}
    In Appendix~\ref{app:reg}, we vary the amount of regularization applied on the MIMIC, DrugLib, and RSNA datasets.
    We find that applying a moderate amount of regularization between 0.2 and 0.6 can lead to improved performance.
    Even when the information matrix is set to identity, i.e., $\lambda = 1$, performance on the DrugLib datasets is still reasonable. Note that for all other experiments, we do not apply any regularization. 

\textbf{Amount of Buyer Data.}
    In Appendix~\ref{app:buyer-data}, we vary how many buyer test datapoints are simultaneously optimized over on Gaussian, MIMIC, and RSNA datasets.
    While all buyer and seller data is sampled from the same distribution, the number of buyer datapoints still affects the optimization procedure. 
    In general, we find that increasing the number of datapoints in the ``test batch'' increases test errors.
    Therefore, we recommend keeping the number of test datapoints in the buyer's ``query'' between 1--8 for each data acquisition. 

\textbf{Number of Steps.}
    In Appendix~\ref{app:steps}, we vary the number of optimization steps in our method on the Gaussian and RSNA datasets. 
    We find that more iterations generally improve prediction performance. 
    Intuitively, one expects that selecting $T$ points requires at least $T$ steps of iterative optimization. 
    We recommend setting the number of steps to be 2--5 times the desired budget for homogeneous costs. 

\textbf{Convex versus Iterative Optimization}
    In Appendix~\ref{app:convex}, we compare the iterative optimization procedure against a convex optimization solver~\cite{diamond2016cvxpy}.
    We find that our iterative approach results in several orders of magnitude speedup while maintaining similar levels of test error.

\textbf{Finetuning versus Linear Probe.}
    In Figure~\ref{fig:bert-finetune}, we evaluate fine-tuning versus linear probing for datapoints selected using DAVID and random selection.  
    We find that using DAVID for fine-tuning performs similarly to linear probing results on DrugLib with BERT~\cite{devlin2018bert}.
    % \vspace*{-1em}
    % This agrees with other work finding that feature functions of data effectively y

% \subsection{Main results}

% \begin{figure}[ht]
%     \centering
%     \includegraphics[width=0.6\columnwidth]{figures/bone-budget-12k.eps}\vspace{-1em}
%     \caption{\textbf{Comparing error under budget constraints on Bone Age X-ray dataset.} Costs are sampled from $c \in \{1,2,3,4,5\}$ with a squared cost function $h(c) = c^2$. Chosen from seller 12,000 points and averaged over 100 buyer test points.}
%     \label{fig:bone-budget}
% \end{figure}

% \begin{figure}[ht]
%     \centering
%     \includegraphics[width=0.6\columnwidth]{figures/gauss-baseline-runtime-1000.eps}\vspace{-1em}
%     \caption{\textbf{Our method has a fast optimization runtime compared with other data selection methods.} We compare data selection time on 1,000 Gaussian data points in 10 dimensions and average over 100 separate optimization runs.}
%     \label{fig:runtime}
% \end{figure}

\section{Discussion}\label{sec:discussion}
% \vspace{-1em}
% \todo[]{Discuss prior optimal design literature and clarify novelty}
    Current data valuation approaches that rely on centralized data access are not well-suited to data market settings~\cite{chen2023data}.
    While there exists other work in validation-free data valuation~\cite{xu2021validation,amiri2023fundamentals}, these methods are not adaptive to the test dataset and are not theoretically grounded like ours.
    Our method (DAVED) outperforms current SOTA data valuation techniques achieving both lower test error and faster runtime. 
    Moreover, a major advantage of our method is that it is amendable to federated optimization requiring $O(d)$ communication per round, making it well-suited for decentralized data markets, unlike other methods that require seller data to be centralized in order to repeatedly train models to estimate data value.
    % Because data is easily duplicated, a seller would not allow access before payment.
    Additionally, our method does not require labeled data, whereas other data valuation methods assume that all datapoints come with corresponding ground-truth labels.
    As discussed in Section~\ref{sec:challenges} and Section~\ref{sec:prior}, the existing paradigm of valuing data with a validation set is suboptimal.
    Incidentally, the second-best performing method, Data OOB~\cite{kwon2023data}, is the only other method that does not use a validation set.

    \textbf{Limitations.} 
    However, our algorithm comes with some limitations that form exciting directions for future work.
    % While our approach could be generalized to other prediction tasks, so far, we have only investigated regression settings; future work will extend to classification tasks. 
    Our approach currently communicates every step. Instead, integrating local steps like in FedAvg~\cite{mcmahan2017communication} or Scaffold~\cite{karimireddy2020scaffold} would decrease communication costs. 
    % The initial computation of the information matrix requires centralized data access by the platform. While initializing with identity matrix alleviates this issue, it may not work reliably in all settings and datasets. 
    Further, integrating differential privacy techniques would provide formal privacy guarantees to the buyers and sellers~\cite{dwork2006differential}.
    % \todo{mention more limitations}
    
% \section{Conclusions}
%     In many practical scenarios involving data markets, a buyer may not want to acquire all available data for sale but only a subset of data that is most relevant to their prediction problem. 
%     In this paper, we focus on the data selection problem for a buyer who wants to acquire training data to predict their test data. 
%     Informed by theoretical insights from optimal experimental design, our method directly optimizes data selection using the buyer's data.
%     We find that compared to current data selection methods that rely on a validation set to measure data value, our approach achieves lower prediction error and is computationally efficient, making it well-suited for decentralized data market settings. 

\section*{Acknowledgments}
C.L. is funded by the National Science Foundation Graduate Research Fellowship Program. 
S.P.K. and M.J. were funded by the European Union (ERC-2022-SYG-OCEAN-101071601).
P.V. was partially supported by the ADIA Lab Fellowship.
% This material is based upon work supported by the National Science Foundation Graduate
% Research Fellowship Program under Grant No. (NSF grant number). 
% Any opinions, findings, and conclusions or recommendations expressed in this material are those of the author(s) and do not necessarily reflect the views of the National Science Foundation.

\bibliography{references}
\bibliographystyle{plainnat}

%%%%%%%%%%%%%%%%%%%%%%%%%%%%%%%%%%%%%%%%%%%%%%%%%%%%%%%%%%%%

\appendix
\newpage
% \section{Appendix / supplemental material}
\setlength\abovecaptionskip{4pt}
\setlength\belowcaptionskip{4pt}

\section{Proof of Theorem~\ref{thm:data-shapley}}
\begin{theorem}\label{thm:data-shapley-formal}
Let $w^*$ denote the solution of Problem~\eqref{eq:platform_obj} and let $\hat w$ denote the solution of Problem~\eqref{eq:data_shapley_opt}. Let the data $Z^{\val}, Z^{\test}$ are drawn i.i.d. from the distribution $\D_{X,Y}$ where $\D_{X}$ is supported on $B_R^d$ (zero-centered ball with radius $R$ in $\R^d$), and $Y = \theta^\top X + \eta$ where $\eta$ is independent zero-meaned noise with variance $\sigma^2$. Suppose $\D_X$ and the training data $X^{\train}$ is supported on $B_R^d$ (zero-centered ball with radius $R$ in $\R^d$) and $l$ is square loss, then there exist numerical constants, $c_1,c_2,c_3$, such that:
\begin{enumerate}
    \item With probability at least $1-\exp\left(-c_1 n_{\val}/R^2\right)$,
\begin{align}
    \inf_{\hat \theta} \sup_{\D_{Y|X}, X^{\train}} \E_{X^{\test}}\left[\cL(\mathbf{\hat w}) - \cL(\mathbf{w}^*)\right] \geq \frac{c_2\sigma^2d}{n_{\val}}. \notag
\end{align}
\item If there exists $\kappa > 0$ such that $\lambda(\E_{x \sim \D_X}[x^{\otimes 4}]) \leq \kappa \cdot \lambda\left(\E_{x \sim \D_X}[x^{\otimes 2}]^{\otimes 2}\right)$ (here $\lambda$ denotes the largest eigenvalue and $\otimes$ denotes the outer product), then for any training algorithm used by the platform, with probability at least $0.99 -\exp\left(-c_1n_{\val}/R^2\right) - \frac{c_3\kappa}{\kappa + m}$, we have
\begin{align}
    % \inf_{\hat w} \sup_{\D_{Y|X},X^{\train}} \E_{Y^{\train}}\left[\cL(\hat w) - \cL(w^*)\right] \geq \frac{c_2\sigma^2d}{n_{\val}}.
    \sup_{\D_{Y|X},X^{\train}} \cL(\mathbf{\hat w}) - \cL(\mathbf{w}^*) \geq \frac{c_2\sigma^2d}{n_{\val}}. \notag
\end{align}
\end{enumerate}
\end{theorem}

\begin{proof}
Let $f_\theta(x) = \theta^\top x$ and $\mathcal{E} = N(0,\sigma^2)$. Define the parameter space resulting from the training algorithm:
\begin{align*}
    \Theta = \left\{\hat \theta(\mathbf{w}): \mathbf{w} \in \Delta([m])\right\}.
\end{align*}
When $n$ is sufficiently large, Lemma~\ref{lem:packing} implies that there exists $\{\theta_1,\theta_2,\dots,\theta_K\} \subset \Theta$ such that
\begin{align*}
    \left\|\theta_i ^\top X^{\val}\right\|_2 \lesssim &~ \delta \sqrt{n_{\val}}, ~\forall i \in [K]\\
    \left\|(\theta_i - \theta_j)^\top X^{\val}\right\|_2 \asymp &~ \delta \sqrt{n_{\val}}, ~\forall i<j \in [K].
\end{align*}
Let $\PP_{\theta_i}$ denote the conditional distribution $\D_{y|x}$ of the target when the underlying model parameter is $\theta_i$, it then follows that
\begin{align*}
    \KL(\PP_{\theta_i}\|\PP_{\theta_j}) \lesssim \frac{n\delta^2}{\sigma^2}.
\end{align*}
Applying Lemma~\ref{lem:fano}, we obtain that
\begin{align}\label{eq:lower-bound-val}
    \inf_{\hat \theta} \sup_{\D_{Y|X},X^{\train}} \E\left[\frac{1}{n_{\val}}\left\|(\hat \theta(\mathbf{\hat w}) - \theta^*)^\top X^{\val}\right\|_2^2\right] \gtrsim \frac{c_2\sigma^2d}{n_{\val}}.
\end{align}
Now define $\Sigma = \E_{x \sim \D_X}[xx^\top]$, by Lemma~\ref{lem:matrix-chernoff}, we have that with probability at least $1-\exp(-\Omega(n_{\val}/R^2))$,
\begin{align}\label{eq:val-concentration}
    \frac{1}{2} \Sigma \preceq X^{\val}(X^{\val})^\top  \preceq 2 \Sigma.
\end{align}
Notice that
\begin{align*}
    \E_{X^{\test}}\left[\cL(\mathbf{\hat w}) - \cL(\mathbf{w}^*)\right] = \left\|(\hat \theta(\mathbf{\hat w}) - \theta^*)^\top\right\|_\Sigma^2.
\end{align*}
Therefore, under the event of Eq.~\eqref{eq:lower-bound-val}, Eq.~\eqref{eq:val-concentration} implies that
\begin{align*}
    \inf_{\hat \theta} \sup_{\D_{Y|X},X^{\train}} \E_{X^{\test}}\left[\cL(\mathbf{\hat w}) - \cL(\mathbf w^*)\right] \geq &~ \inf_{\hat \theta} \sup_{\D_{Y|X},X^{\train}} \frac{1}{2}\E\left[\frac{1}{n_{\val}}\left\|(\hat \theta(\hat w) - \theta^*)^\top X^{\val}\right\|_2^2\right] \\
    \gtrsim &~ \frac{c_2\sigma^2d}{n_{\val}}
\end{align*}
This establishes the first inequality.

For the second inequality, we have that under the condition $\lambda(\E_{x \sim \D_X}[x^{\otimes 4}]) \leq \kappa \cdot \lambda\left(\E_{x \sim \D_X}[x^{\otimes 2}]^{\otimes 2}\right)$, the following holds
\begin{align*}
    \Var_{x \sim D_X}\left(\left\langle\hat \theta(\mathbf{\hat w}) - \theta^*, x\right\rangle^2\right) = &~ \E\left[\left\langle\hat \theta(\mathbf{\hat w}) - \theta^*, x\right\rangle^4\right] - \E\left[\left\langle\hat \theta(\mathbf{\hat w}) - \theta^*, x\right\rangle^2\right]^2\\
    = &~ \E\left[\left\langle(\hat \theta(\mathbf{\hat w}) - \theta^*)^{\otimes 4}, x^{\otimes 4}\right\rangle\right] - \E\left[\left\langle\hat \theta(\mathbf{\hat w}) - \theta^*)^{\otimes 2}, x^{\otimes 2}\right\rangle^2\right]\\
    = &~ \left\langle(\hat \theta(\mathbf{\hat w}) - \theta^*)^{\otimes 4}, \E\left[x^{\otimes 4}\right] - \E\left[x^{\otimes 2}\right]^{\otimes 2} \right\rangle\\
    \leq &~ (\kappa - 1) \cdot \E\left[\left\langle\hat \theta(\mathbf{\hat w}) - \theta^*, x\right\rangle^2\right]^2.
\end{align*}
By Lemma~\ref{lem:paley-zygmund}, for any $\theta$ and $\theta^*$, we have that with probability at least $0.99 - O\left(\frac{\kappa}{\kappa + m}\right)$,
\begin{align*}
    \cL(\mathbf{\hat w}) - \cL(\mathbf{w}^*) = &~ \frac{1}{m}\left\|(\theta - \theta^*)^\top X^{\test}\right\|_2^2 \\
    = &~ \frac{1}{m}\sumi \left\langle\theta - \theta^*, x^{\test}_i\right\rangle^2 \\
    \geq &~ 0.0001\cdot\left\|(\hat \theta(\mathbf{\hat w}) - \theta^*)^\top\right\|_\Sigma^2\\
    \gtrsim &~ \frac{c_2\sigma^2d}{n_{\val}}.
\end{align*}
Combining this and the first inequality by union bound, we establish the second inequality.
\end{proof}

\subsection{Supporting Lemma}

\begin{lemma}[Metric entropy, \citet{wainwright2019high}]\label{lem:packing}
Let $\|\cdot\|$ denote the Euclidean norm on $\mathbb{R}^d$ and let $\mathbb{B}$ be the unit balls (i.e., $\mathbb{B} = \{\theta \in \mathbb{R}^d | \|\theta\| \leq 1\}$). Then the $\delta$-covering number of $\mathbb{B}$ in the $\|\cdot\|$-norm obeys the bounds
\begin{align*}
d \log\left(\frac{1}{\delta}\right) \leq \log N(\delta; \mathbb{B}, \|\cdot\|) \leq d \log\left(1 + \frac{2}{\delta}\right).
\end{align*}
\end{lemma}

\begin{lemma}[Fano's inequality, \citet{cover1999elements}]\label{lem:fano}
When $\theta$ is uniformly distributed over the index set $[M]$, then for any estimator $\hat \theta$ such that $\theta \rightarrow Z \rightarrow \hat \theta$
\begin{align*}
\mathbb{P}[\hat \theta(Z) \neq \theta] \geq 1 - \frac{I(Z; \theta) + \log 2}{\log M}.
\end{align*}
\end{lemma}

\begin{lemma}[Matrix-Chernoff bound, \citet{tropp2012user}]\label{lem:matrix-chernoff}
Consider an independent sequence $\{X_i\}_{i=1}^k$ of random, self-adjoint matrices in $M_n$ satisfying $X_i \geq 0$ and $\lambda_{\text{max}}(X_i) \leq R$ almost surely, for each $i \in \{1, \ldots, k\}$. Define
\begin{align*}
\mu_{\text{min}} &:= \lambda_{\text{min}}\left(\sum_{i=1}^k \mathbb{E}X_i\right), \\
\mu_{\text{max}} &:= \lambda_{\text{max}}\left(\sum_{i=1}^k \mathbb{E}X_i\right).
\end{align*}
Then,
\begin{align*}
\mathbb{P}\left\{\lambda_{\text{max}}\left(\sum_{i=1}^k X_i\right) \geq (1 + \delta)\mu_{\text{max}}\right\} &< n\left(\frac{\mathrm{e}^\delta}{(1 + \delta)^{1+\delta}}\right)^{\frac{\mu_{\text{max}}}{R}} \quad \text{for } \delta \geq 0; \\
\mathbb{P}\left\{\lambda_{\text{min}}\left(\sum_{i=1}^k X_i\right) \leq (1 - \delta)\mu_{\text{min}}\right\} &< n\left(\frac{\mathrm{e}^{-\delta}}{(1 - \delta)^{1-\delta}}\right)^{\frac{\mu_{\text{min}}}{R}} \quad \text{for } \delta \in [0, 1].
\end{align*}
\end{lemma}

\begin{lemma}[Paley–Zygmund inequality, \citet{paley1932note}]\label{lem:paley-zygmund}
If $Z$ is a random variable with finite variance and $Z \geq 0$ almost surely, then
\begin{align*}
\mathbb{P}(Z > \theta ~ \mathbb{E}[Z]) \geq (1 - \theta)^2 \frac{\mathbb{E}[Z]^2}{\mathbb{E}[Z^2]}, ~\forall \theta \in (0, 1).
\end{align*}
\end{lemma}

\section{Convergence rate of Frank-Wolfe (Proof of Theorem~\ref{thm:herding}}\label{app:convergence}
\textbf{Setup.}
We mostly follow the proof technique in \citet{bach2012herding} and \citet{jaggi2013revisiting}. Recall the optimization problem for the optimal design loss Eq.~\eqref{eq:design_opt}:
\begin{align*}
    \min_{\mathbf{w} \in \D} \mathcal{L}(\mathbf{w}).
\end{align*}
Here, $\D$ is the scaled simplex defined by the constraints $\mathbf{w} \in \R_{\geq 0}^n$ and $\sumj c_j w_j \leq B$. The Frank-Wolfe update on this function is then: For $t = 1,2,\dots$, repeatedly perform the following steps
\begin{itemize}
    \item Compute $\mathbf{s}_t = \arg\max_{\mathbf{u} \in \D} \langle \nabla \mathcal{L}(\mathbf{w}_t), \mathbf{w}_t-\mathbf{u} \rangle$.
    \item update $\mathbf{w}_{t+1} = (1-\alpha_t) \mathbf{w}_t + \alpha_t \mathbf{s}_t$
\end{itemize}
Note that this update procedure is identical to the updates in Eq.~\ref{eq:FW-herding-update}.
We can then define the duality gap as
\begin{align*}
    g(\mathbf{w}) = \sup_{\mathbf{s} \in \D} \langle \mathbf{w}-\mathbf{s},\nabla \mathcal{L}(\mathbf{w})\rangle.
\end{align*}
We also define the curvature constant
\begin{align*}
    C_l = \sup_{\substack{\mathbf{s},\mathbf{w}\in \D \\ \gamma \in (0,1) \\ \mathbf{u} = (1-\gamma)\mathbf{w} + \gamma \mathbf{s}}} \frac{2}{\gamma^2}\left(\mathcal{L}(\mathbf{u}) - \mathcal{L}(\mathbf{w}) - \langle \mathbf{u}-\mathbf{w}, \nabla \mathcal{L}(\mathbf{w})\rangle \right).
\end{align*}
We assume that the curvature constant is finite, i.e., $C_l < \infty$. This is true for $\mathcal{L}(\mathbf{w})$ as long as both the algorithm and the true optimum are bounded away from the boundary of $\D$ --- see detailed discussions on this in \citet{ahipacsaouglu2013modified}. A better analysis might be able to avoid this assumption, e.g.~\citet{zhao2023analysis} use certain homogeneity properties of $\mathcal{L}(\mathbf{w})$ to derive better assumption-free convergence rates for the FW method on D-optimal experiment design. We leave the question of adapting these results to our setting (V-optimal experiment design) for a challenging future work.

\begin{lemma}[Lemma 5, \cite{jaggi2013revisiting}]\label{lem:one_step_fw}
For any $\alpha \in (0,1)$,
\begin{align*}
    \mathcal{L}(\mathbf{w}_{t+1}) \leq \mathcal{L}(\mathbf{w}_{t}) - \alpha_t g(\mathbf{w}_t) + \frac{\alpha_t^2}{2}C_l.
\end{align*}
\end{lemma}

\begin{theorem}
In the Frank-Wolfe algorithm, our update algorithm in Eq.~\ref{eq:FW-herding-update} uses $\alpha_t = \frac{1}{t+1}$. For this, we have
\begin{align*}
    \mathcal{L}(\mathbf{w}_{t}) \leq \min_{\mathbf{w} \in \D}\mathcal{L}(\mathbf{w}) + \frac{C_l(1 + \log t)}{2t}.
\end{align*}
\end{theorem}
\begin{proof}
Define $h(\mathbf{w}) = \mathcal{L}(\mathbf{w})  - \min_{\mathbf{w} \in \D}\mathcal{L}(\mathbf{w})$. 
Using Lemma~\ref{lem:one_step_fw},
\begin{align*}
    h(\mathbf{w}_t) \leq &~ h(\mathbf{w}_{t-1}) - \alpha_{t-1} g(\mathbf{w}_{t-1}) + \frac{\alpha_{t-1}^2}{2}C_l\\
    \leq &~ h(\mathbf{w}_{t-1}) - \alpha_{t-1} h(\mathbf{w}_{t-1}) + \frac{\alpha_{t-1}^2}{2}C_l\\
    = &~ (1-\alpha_{t-1})h(\mathbf{w}_{t-1}) + \frac{\alpha_{t-1}^2}{2}C_l
\end{align*}
Here we used the convexity of $\mathcal{L}$ as follows: 
\begin{align*}
g(\mathbf{w}_{t-1}) &= \sup_{\mathbf{s} \in \D} \langle \mathbf{w}_{t-1}-\mathbf{s},\nabla \mathcal{L}(\mathbf{w}_{t-1})\rangle\\
&\geq \langle \mathbf{w}_{t-1}- \mathbf{w}^\ast,\nabla \mathcal{L}(\mathbf{w}_{t-1})\rangle\\
&\geq \mathcal{L}(\mathbf{w}_{t-1}) - \min_{\mathbf{w} \in \D}\mathcal{L}(\mathbf{w})\,.
\end{align*}
Continuing our derivation, recall we have
\begin{align*}
    h(\mathbf{w}_t) &\leq (1-\alpha_{t-1})h(\mathbf{w}_{t-1}) + \frac{\alpha_{t-1}^2}{2}C_l\\
    &= \frac{t-1}{t}h(\mathbf{w}_{t-1}) + \frac{1}{2 t^2}C_l
    \quad\leq\quad  \frac{t-2}{t}h(\mathbf{w}_{t-1}) \,.
\end{align*}
Thus we have
\[
t\cdot h(\mathbf{w}_t) \leq (t-1)h(\mathbf{w}_{t-1}) + \frac{C_l}{2t}\,.
\]
Unrolling this recursion, we get
\[
t\cdot h(\mathbf{w}_t) \leq \sum_{k=1}^{t-1} \sum  \frac{C_l}{2k} \leq \frac{C_l(1 + \log t)}{2}\,.
\]
This yields the theorem claim.
\end{proof}
Finally, to finish the proof of Theorem~\ref{thm:herding}, note that adding additional constraints only increases the loss, i.e. 
\[
\min_{\mathbf{w} \in \D} \mathcal{L}(\mathbf{w}) \leq \min_{\mathbf{w} \in \D \text{ and } \mathbf{w} \in \{0,1\}^n} \mathcal{L}(\mathbf{w}).
\]
Hence, we showed a $O(\log t/t)$ approximation to the otherwise intractable combinatorial problem. With additional assumptions on the structure of the loss function, one can even show improved quadratic or even exponential approximations~\citep{damla2008linear,bach2012herding}. Finally, we note that Frank-Wolfe is a well-known method to efficiently approximate the optimal experiment design objective~\cite{wynn1970sequential,fedorov2013theory,ahipacsaouglu2013modified} and this also motivates using this approach in practice~\cite{ahipacsaouglu2013modified,amjad2021optimal}.

\section{Experimental Setup}\label{app:setup}
For each buyer test point, we optimize each selection algorithm over the 1,000 seller datapoints and select the highest value data based on the validation set of 100 datapoints (our method and Data OOB do not use the validation set). For each test point, we train a linear regression model on the selected seller points and report test mean squared error (MSE) on the buyer's data and average test error over 100 buyers.

We benchmarked all results on an Intel Xeon E5-2620 CPU with 40 cores and a Nvidia GTX 1080 Ti GPU.
For baseline implementation of data valuation methods, we use the OpenDataVal package~\cite{jiang2023opendataval} version 1.2.1. We use the default hyperparameter settings for all methods with the exception of Data Shapley (changed 100 Monte-Carlo epochs with 10 models per iteration), Influence Subsample (from 1000 to 500 models), and Data OOB (from 1000 to 500 models) to reduce computational runtime.

For the Gaussian dataset, we generate a regression dataset according to the following Python code:
\begin{lstlisting}[language=Python]
import numpy as np

def get_gaussian_data(num_samples=100, dim=10, noise=0.1, costs=None):
    X = np.random.normal(size=(num_samples, dim))
    X /= np.linalg.norm(X, axis=1, keepdims=True)
    if costs is not None:
        X *= costs
    coef = np.random.exponential(scale=1, size=dim)
    coef *= np.sign(np.random.uniform(low=-1, high=1, size=dim))
    y = X @ coef + noise * np.random.randn(num_samples)
    return dict(X=X, y=y, coef=coef, noise=noise, dim=dim, costs=costs)
\end{lstlisting}

The RSNA Pediatric Bone Age Challenge (2017) dataset~\cite{halabi2019rsna} may be downloaded here \url{https://www.rsna.org/rsnai/ai-image-challenge/rsna-pediatric-bone-age-challenge-2017}. We use the training set for our experiments, resulting in 12,611 images in total. Using the following function, each image was embedded through a pre-trained CLIP ViT-B/32 model.
\begin{lstlisting}[language=Python]
import clip
import torch
from PIL import Image

def embed_images(img_paths, device="cuda"):
    model, preprocess = clip.load("ViT-B/32", device=device)
    inference_func = model.encode_image
    embeddings = []
    with torch.inference_mode():
        for img_path in tqdm(img_paths):
            img = Image.open(img_path)
            embedding = inference_func(preprocess(img)[None].to(device))
            embeddings.append(embedding.cpu())
    return torch.cat(embeddings)
\end{lstlisting}

The Fitzpatrick17K~\cite{groh2021evaluating} can be downloaded from here \url{https://github.com/mattgroh/fitzpatrick17k}. Missing or corrupted images were excluded, resulting in 16,536 total images. Some images were annotated with two separate annotation platforms. We averaged the two skin type ratings for these images, resulting in 12 possible labels (0--6 in 0.5 increments). Each image was embedded in a fashion similar to the RSNA Bone Age dataset. 

The MIMIC dataset~\cite{johnson2016mimic} can be assessed here \url{https://physionet.org/content/mimiciii/1.4/}.
The task is to predict the length of stay (LOS) in the number of days a patient stays in the Intensive Care Unit. The dataset contains 51,036 rows with both real-valued and one-hot-encoded attributes with the following names:
\begin{quote}
    ``LOS'' , ``blood'' , ``circulatory'' , ``congenital'' , ``digestive'' , ``endocrine'' ,
    ``genitourinary'' , ``infectious'' , ``injury'' , ``mental'' , ``misc'' , ``muscular'' ,
    ``neoplasms'' , ``nervous'' , ``pregnancy'' , ``prenatal'' , ``respiratory'' , ``skin'' ,
    ``GENDER'' , ``ICU'' , ``NICU'' , ``ADM\_ELECTIVE'' , ``ADM\_EMERGENCY'' , ``ADM\_NEWBORN'' ,
    ``ADM\_URGENT'' , ``INS\_Government'' , ``INS\_Medicaid'' , ``INS\_Medicare'' ,
    ``INS\_Private'' , ``INS\_Self Pay'' , ``REL\_NOT SPECIFIED'' , ``REL\_RELIGIOUS'' ,
    ``REL\_UNOBTAINABLE'' , ``ETH\_ASIAN'' , ``ETH\_BLACK/AFRICAN AMERICAN'' ,
    ``ETH\_HISPANIC/LATINO'' , ``ETH\_OTHER/UNKNOWN'' , ``ETH\_WHITE'' ,
    ``AGE\_middle\_adult'' , ``AGE\_newborn'' , ``AGE\_senior'' , ``AGE\_young\_adult'' ,
    ``MAR\_DIVORCED'' , ``MAR\_LIFE PARTNER'' , ``MAR\_MARRIED'' , ``MAR\_SEPARATED'' ,
    ``MAR\_SINGLE'' , ``MAR\_UNKNOWN (DEFAULT)'' , ``MAR\_WIDOWED''
\end{quote}
Each attribute was min-max scaled to lie in the range [0, 1].

The DrugLib dataset~\cite{misc_drug_review_dataset_(druglib.com)_461} can be downloaded here \url{https://archive.ics.uci.edu/dataset/461/drug+review+dataset+druglib+com}.
The task is to predict overall patient satisfaction with a drug's side effects and effectiveness on a ten-point scale.
We format each review using the following prompt template to feed GPT-2:
\begin{verbatim}
    Benefits: $BENEFITS_REVIEW
    Side effects: $SIDE_EFFECTS_REVIEW
    Comments: $COMMENTS_REVIEW
\end{verbatim}
where \texttt{\$BENEFITS\_REVIEW} is the corresponding portion of the drug review.
\begin{lstlisting}[language=Python]
import torch
from transformers import GPT2Tokenizer, GPT2Model

def embed_text(text_inputs :list[str], device="cuda"):
    tokenizer = GPT2Tokenizer.from_pretrained(model_name)
    model = GPT2Model.from_pretrained(model_name).to(device)
    embeddings = []
    for x in tqdm(text_inputs):
        inputs = tokenizer(x, return_tensors="pt", truncation=True).to(device)
        with torch.no_grad():
            outputs = model(**inputs)
        embeddings.append(
            outputs.last_hidden_state.mean(dim=1).cpu()
        )
    return torch.cat(embeddings)
\end{lstlisting}

% \clearpage

\section{Additional Experiments}\label{app:exp}
\subsection{Budget Results with Heterogeneous Costs}\label{app:budget}

\begin{figure}[H]
    \centering
    \includegraphics[width=\textwidth]{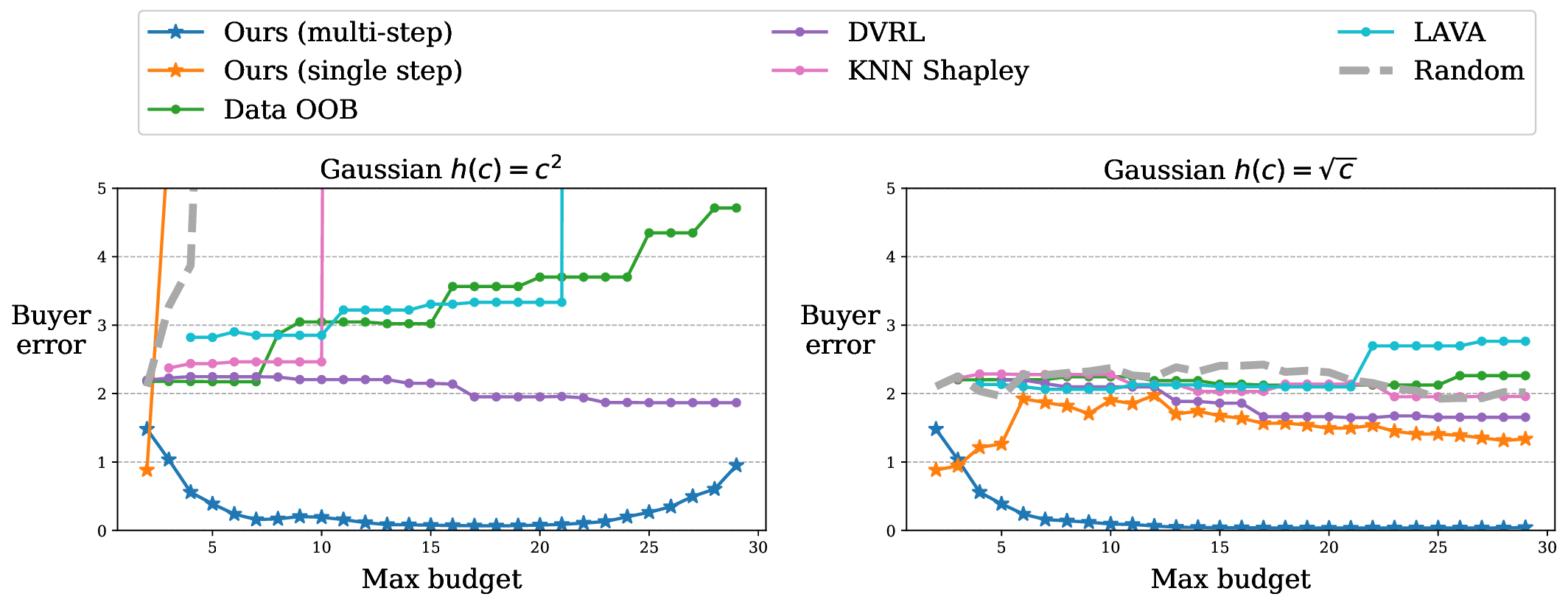}
    \caption{We compare data selection methods on synthetic Gaussian data with 10K seller datapoints under two cost functions. No feature function was used, and results were averaged over 100 buyers. We find that our approach using experimental design (DAVED) is more budget-efficient than other data valuation methods.
    }
    \label{fig:gauss-cost}
\end{figure}

\begin{figure}[H]
    \centering
    \includegraphics[width=\textwidth]{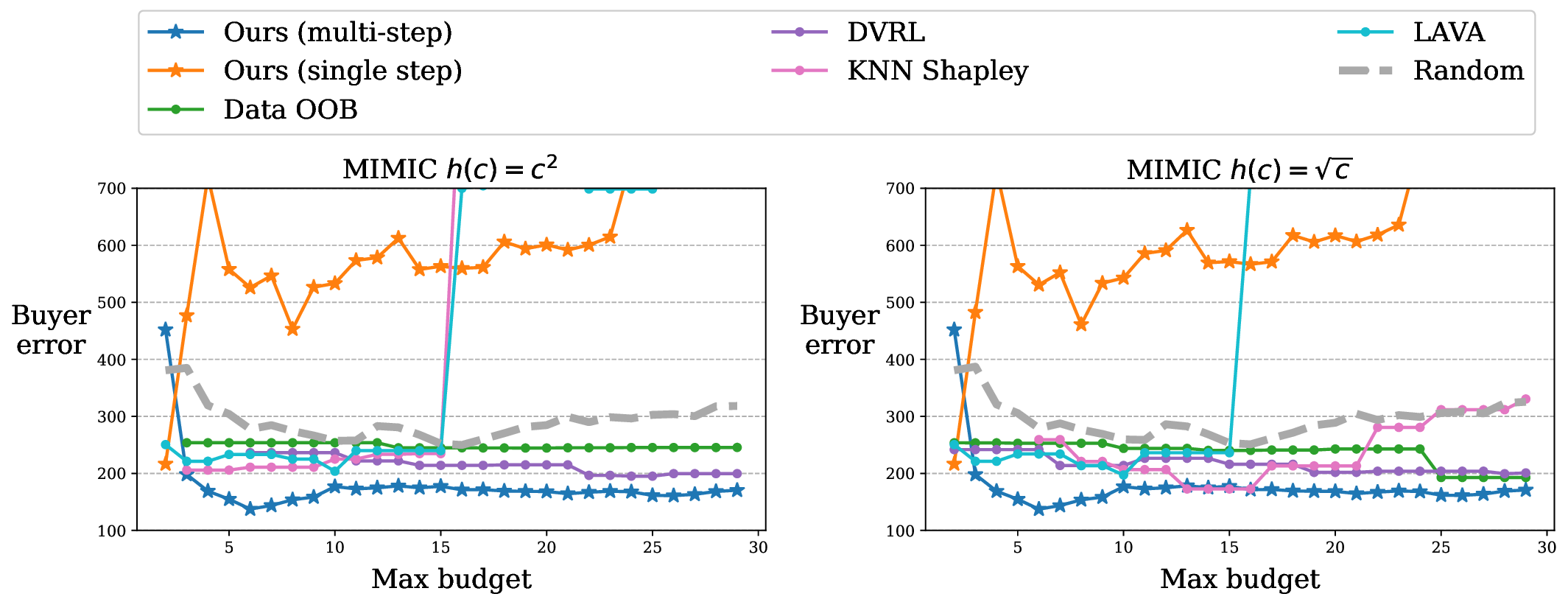}
    \caption{We compare data selection methods on the MIMIC-III dataset with 35,000 seller datapoints under two cost functions. No feature function was used, and results were averaged over 100 buyers. We find that our approach using experimental design (DAVED) is more budget-efficient than other data valuation methods.
    }
    \label{fig:mimic-cost}
\end{figure}

\begin{figure}[H]
    \centering
    \includegraphics[width=\textwidth]{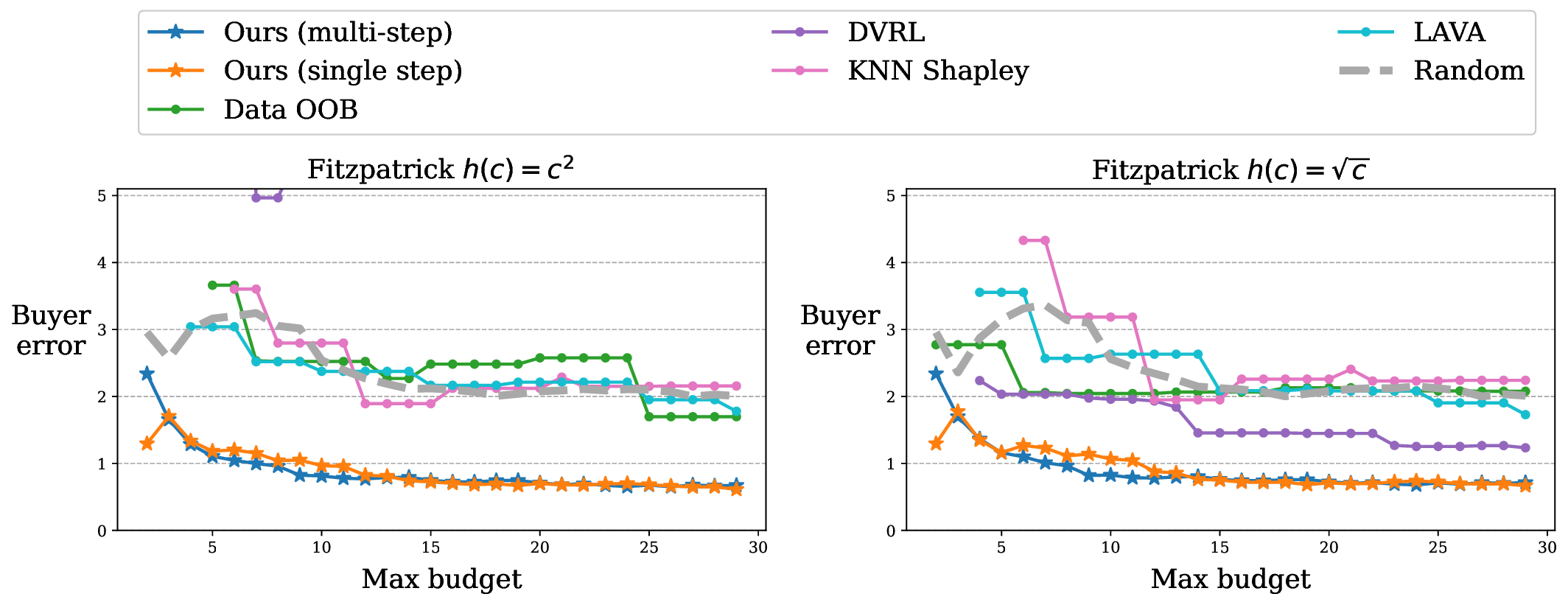}
    \caption{
    We compare data selection methods on the Fitzpatrick dataset with 15,000 seller datapoints under two cost functions. Each image was embedded through CLIP, and the results averaged over 100 buyers. We find that our approach using experimental design (DAVED) is more budget-efficient than other data valuation methods.
    }
    \label{fig:fitz-cost}
\end{figure}

\begin{figure}[H]
    \centering
    \includegraphics[width=\textwidth]{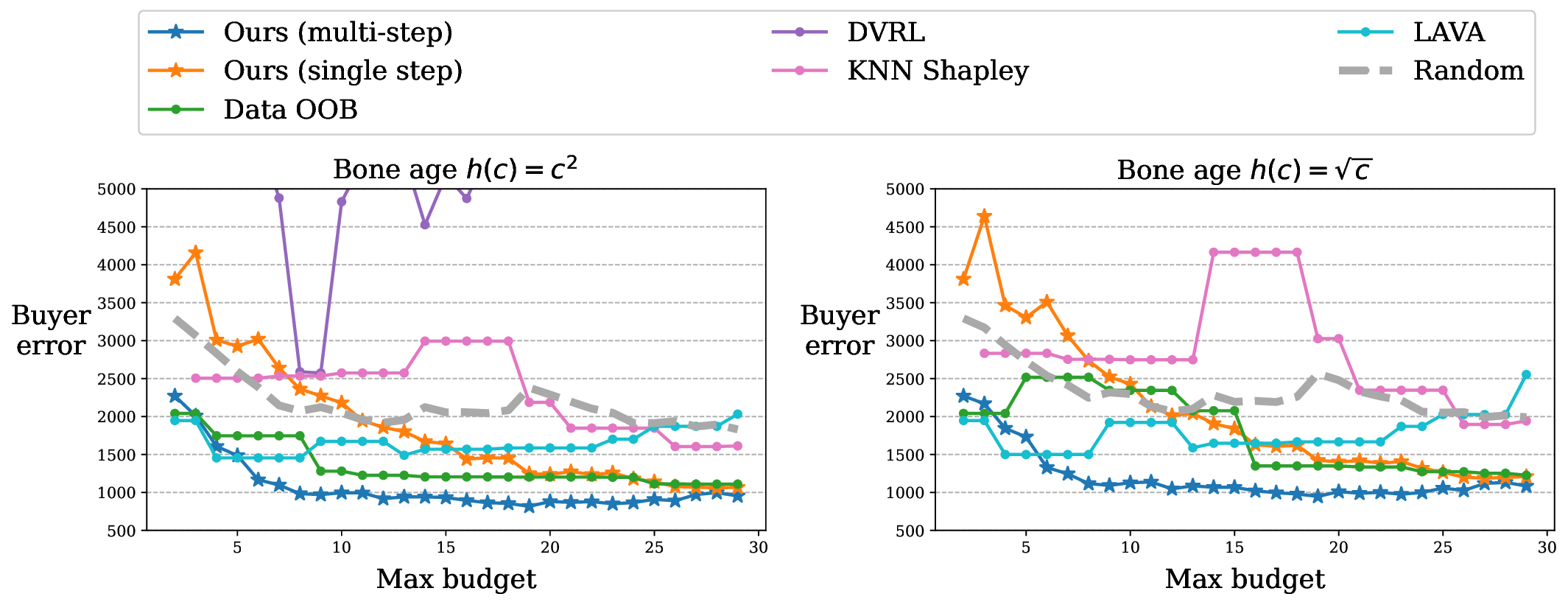}
    \caption{
    We compare data selection methods on the RSNA Bone Age dataset with 12,000 seller datapoints under two cost functions. Each image was embedded through CLIP, and the results averaged over 100 buyers. We find that our approach using experimental design (DAVED) is more budget-efficient than other data valuation methods.
    }
    \label{fig:bone-cost}
\end{figure}

\begin{figure}[H]
    \centering
    \includegraphics[width=\textwidth]{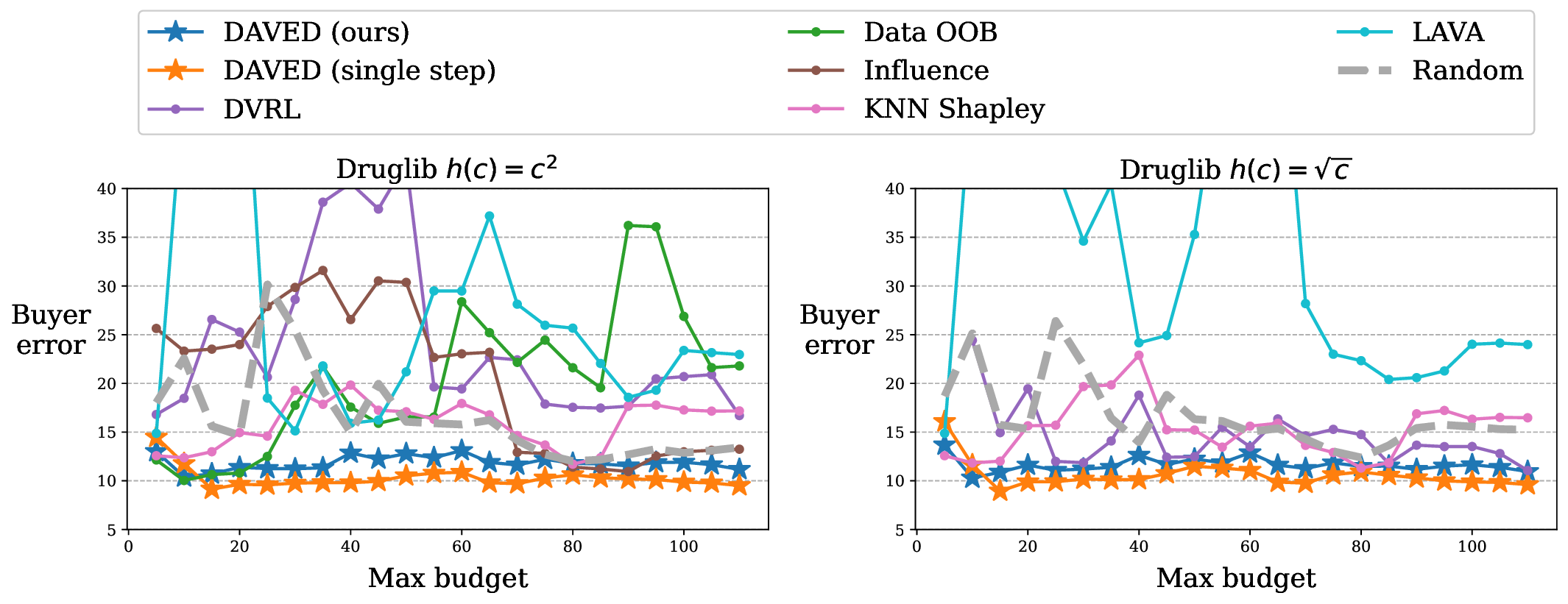}
    \caption{
    We compare data selection methods on the DrugLib reviews dataset with 3,500 seller datapoints under two cost functions. Each image was embedded through GPT2, and the results averaged over 100 buyers. We find that our approach using experimental design (DAVED) is more budget-efficient than other data valuation methods.
    }
    \label{fig:drug-cost}
\end{figure}

% \subsection{Number of Seller Data}\label{app:seller}

% \begin{figure}[H]
%     \centering
%     \includegraphics[width=0.9\textwidth]{figures/mimic-compare-error.eps}
%     \caption{Comparison of buyer prediction error across data selection methods on the MIMIC dataset with 1,000 seller datapoints (left subplot) and 35,000 seller datapoints (right).}
%     \label{fig:mimic-num-seller}
% \end{figure}

% \subsection{Runtime Comparison}\label{app:runtime}
% \begin{figure}[H]
%     \centering
%     \includegraphics[width=0.8\textwidth]{figures/runtime_num.eps}
%     \caption{Runtime comparison with increasing number of datapoints.}
%     \label{fig:runtime-num}
% \end{figure}

% \begin{figure}[H]
%     \centering
%     \includegraphics[width=0.8\textwidth]{figures/runtime_dim.eps}
%     \caption{Runtime comparison with increasing dimensionality.}
%     \label{fig:runtime-dim}
% \end{figure}
\subsection{Regularization}\label{app:reg}

\begin{figure}[H]
    \centering
    \includegraphics[width=0.99\textwidth]{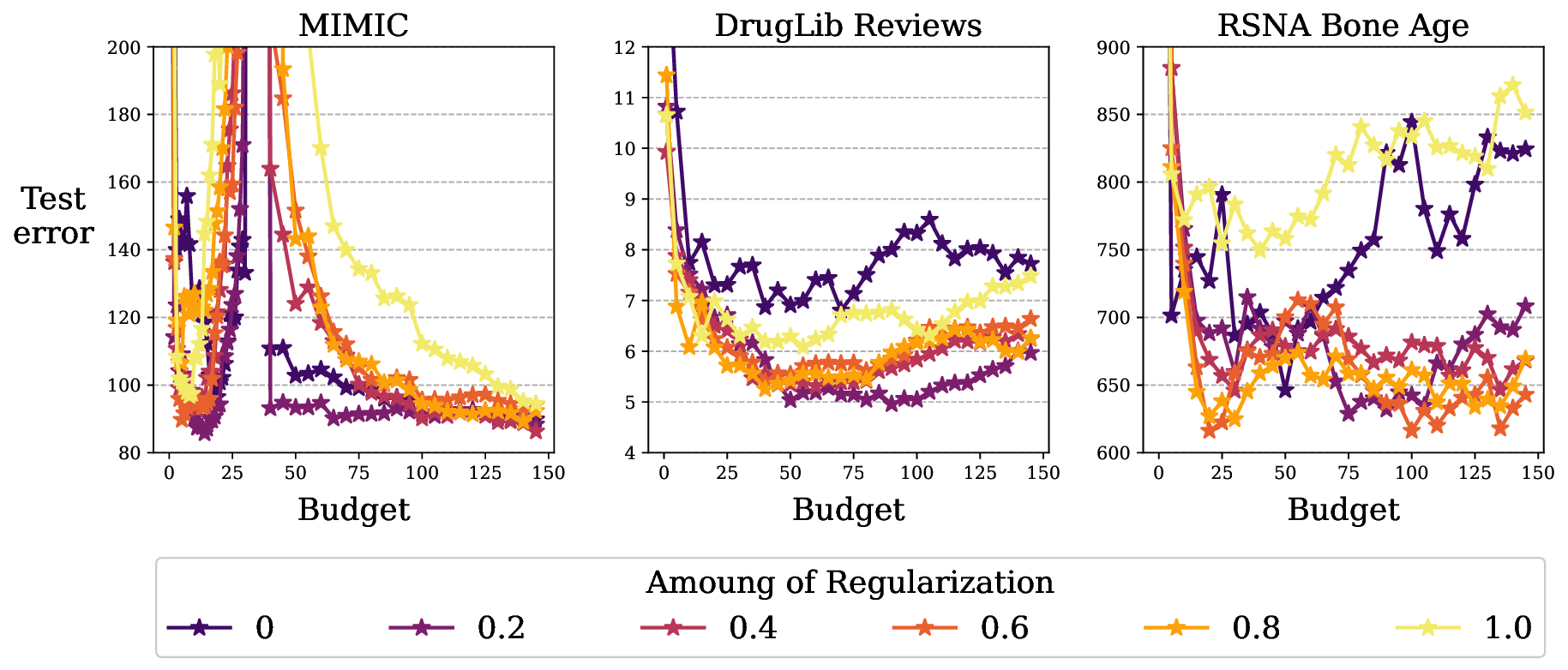}
    \caption{We vary the strength of regularization on the MIMIC, DrugLib, and RSNA Bone Age datasets. In general, a moderate amount can improve performance and stability across budgets. Note that we do not apply regularization in any other experiments.}
    \label{fig:regularization}
\end{figure}

\subsection{Amount of Buyer Data}\label{app:buyer-data}
\begin{figure}[H]
    \centering
    \includegraphics[width=0.99\textwidth]{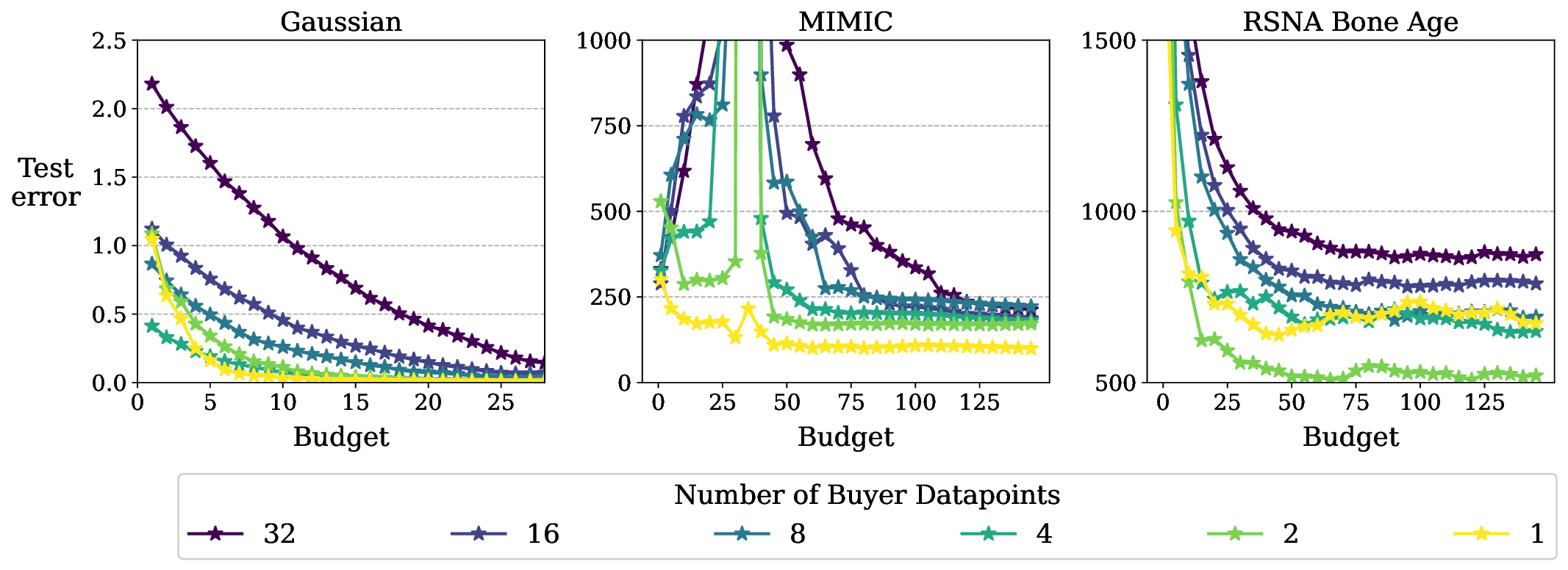}
    \caption{We vary the amount of buyer test datapoints on Gaussian, MIMIC, and RSNA Bone Age datasets. In general, we find that increasing the number of datapoints being optimized results in increased test error.}
    \label{fig:vary-buyer-points}
\end{figure}

\subsection{Number of Iteration Steps}\label{app:steps}
\begin{figure}[H]
    \centering
    \includegraphics[width=0.99\textwidth]{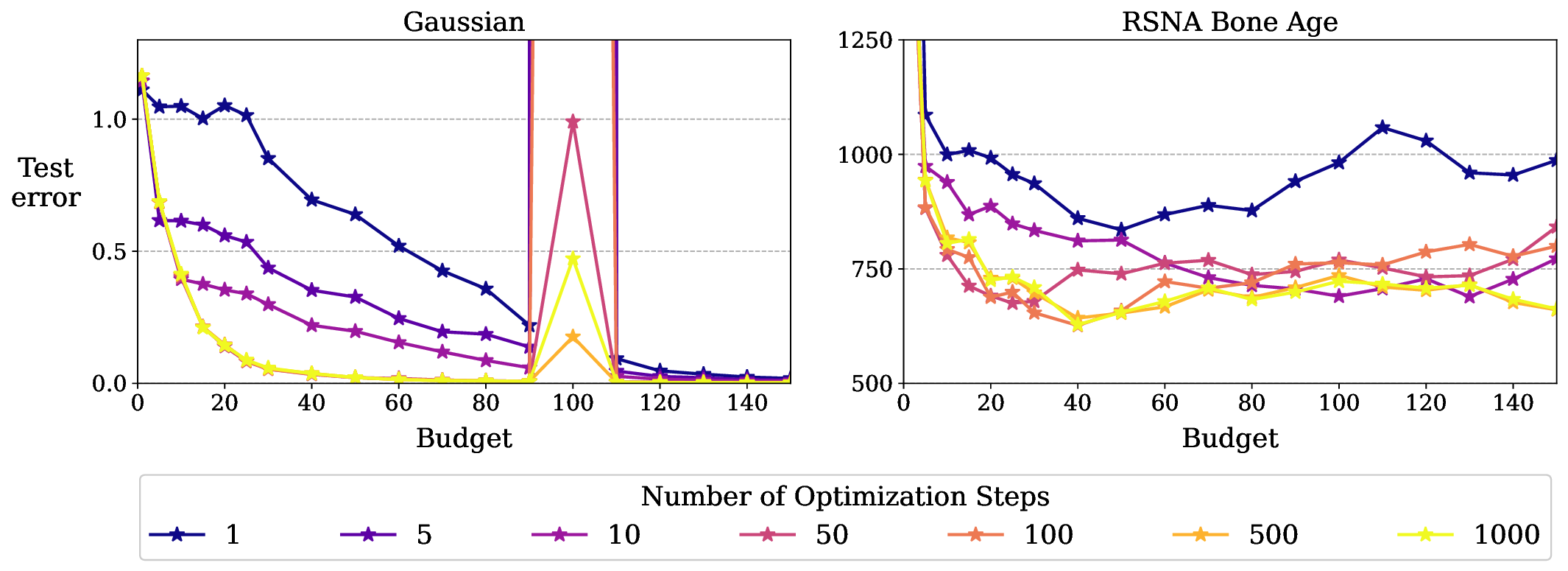}
    \caption{We vary the number of optimization steps on Gaussian and RSNA Bone Age datasets. In general, we recommend 2--5 times the number of optimization steps as the total budget, or the number of datapoints selected in the case where each datapoint has a cost of 1.}
    \label{fig:vary-steps}
\end{figure}

\subsection{Feature Function}\label{app:feature}
\begin{figure}[H]
    \centering
    \includegraphics[width=0.5\textwidth]{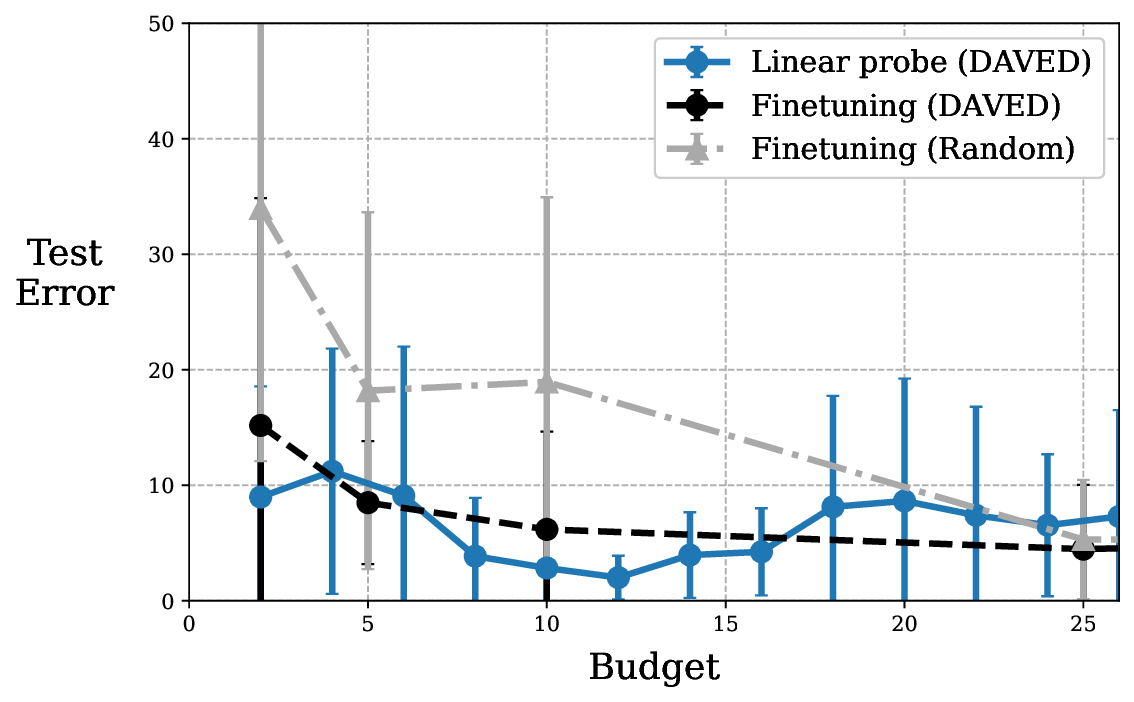}
    \caption{We finetune a Bert model on data selected using DAVED versus randomly selected datapoints on the DrugLib dataset. We find that the performance of fine-tuning corresponds with that of the linear probe, which is used for the other experiments. This provides empirical justification for our choice of feature function.}
    \label{fig:bert-finetune}
\end{figure}

\subsection{Iterative versus Convex optimization}\label{app:convex}
\begin{figure}[H]
    \centering
    \includegraphics[width=0.9\textwidth]{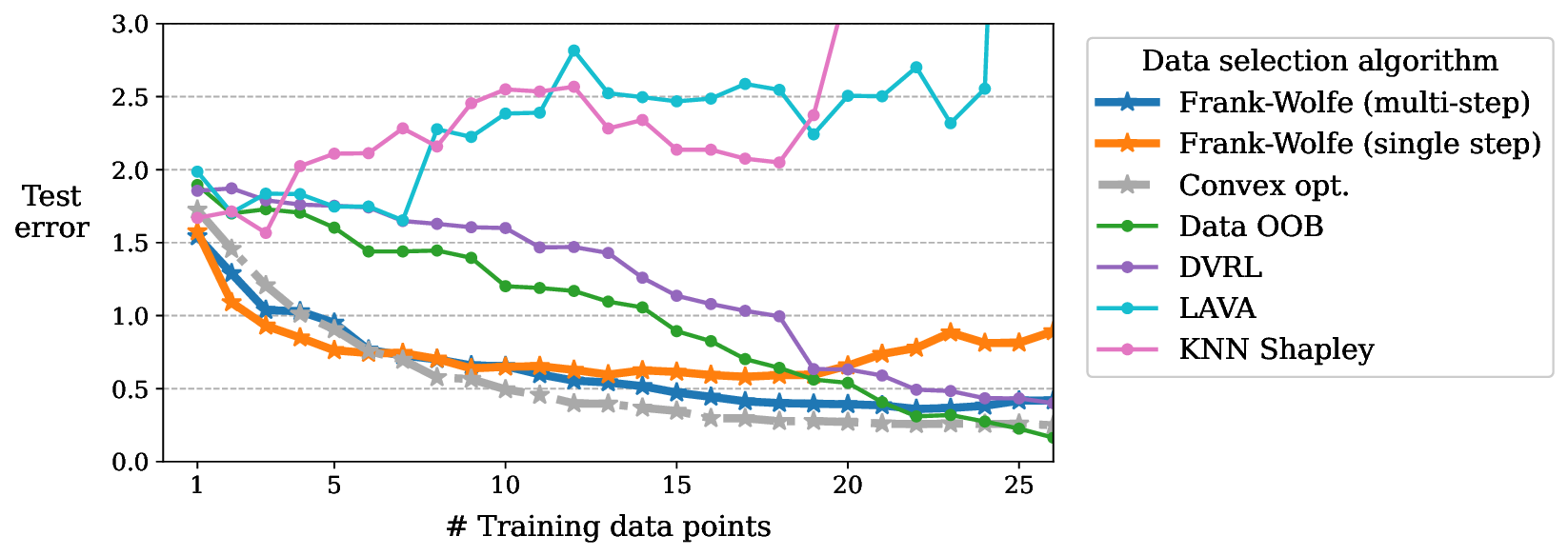}
    \caption{We compare the prediction error of Frank-Wolfe optimization with Convex optimization on 1000 datapoints sampled from 30-dimensional Gaussian. We find that our iterative optimization procedure generally approximates the convex solver in terms of test error.}
    \label{fig:compare-opt-error}
\end{figure}

\begin{figure}[H]
    \centering
    \includegraphics[width=0.7\textwidth]{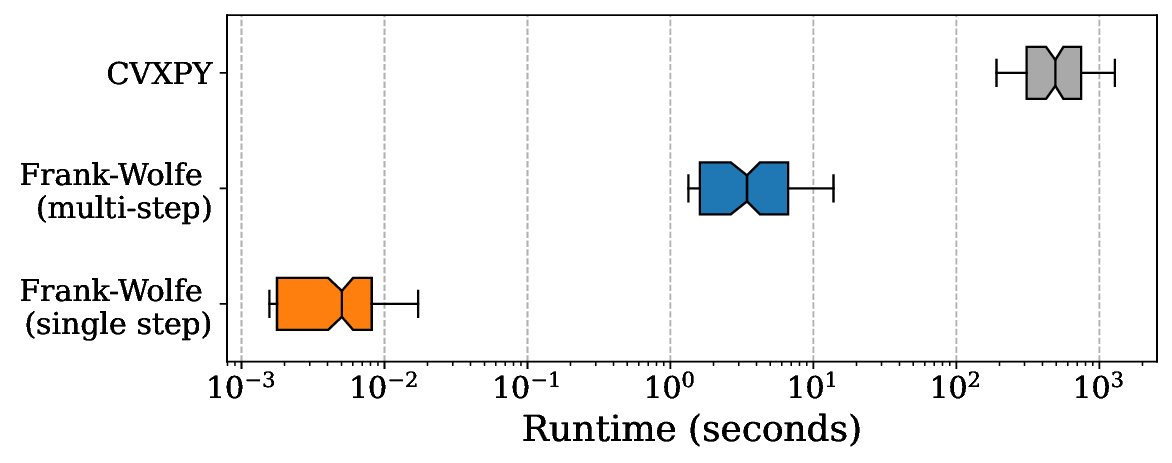}
    \caption{We compare the runtime of Frank-Wolfe optimization with Convex optimization on 1000 datapoints sampled from 30-dimensional Gaussian. We find that our iterative optimization procedure is much faster to optimize than the convex solver.}
    \label{fig:compare-opt-speed}
\end{figure}

\section{Broader Impacts}\label{app:ethics}
    We believe that AI developers face important ethical questions when acquiring data to for AI development. 
    Class-action lawsuits have been filed against several AI companies for their data collection practices surrounding data consent and compensation from data owners.
    Our work is situated in the context of decentralized data marketplaces, which promise an alternative model of data acquisition and ownership.
    Existing data brokers have aggregated vast amounts of data from numerous sources to be sold off in opaque transactions, often without the data users' knowledge or consent. 
    Questions revolving around data ownership and what rights should be circumscribed are also still under debate.
    More public dialog about the potential benefits and harms of data sharing is needed. 
    
    Our work presents a scalable and decentralized approach to selecting the most useful datapoints for a buyer.
    Selecting a subset of data is more economical and may result in less overall privacy loss than broad, indiscriminate data access~\cite{nielsen99whose}.
    Our method allows for a more targeted approach to data acquisition that is more efficient for both the data buyer and the data seller.
    However, we also hope to advance the idea that more decentralized models of data governance will result in greater transparency and control by individual data owners over how data is shared and used.
    Centralizing all data with the broker can result in undesirable privacy and security risks, such as data breaches.
    In contrast, decentralized data marketplaces may be more robust and transparent. 
    Bypassing intermediate data brokers will enhance privacy and increase market efficiency. Transaction costs can be reduced, and revenue can be directly captured by data producers. This will enable a greater number of data transactions and sustain more types of markets.
    However, several other societal and technical challenges must be addressed in order to develop these decentralized data markets responsibly and safely.

% Optionally include supplemental material (complete proofs, additional experiments and plots) in appendix.
% All such materials \textbf{SHOULD be included in the main submission.}

%%%%%%%%%%%%%%%%%%%%%%%%%%%%%%%%%%%%%%%%%%%%%%%%%%%%%%%%%%%%

\newpage
\section*{NeurIPS Paper Checklist}

\begin{enumerate}

\item {\bf Claims}
    \item[] Question: Do the main claims made in the abstract and introduction accurately reflect the paper's contributions and scope?
    \item[] Answer: \answerYes{} 
    \item[] Justification: We believe that the abstract and introduction accurately reflect the paper's contribution and scope. 
    \item[] Guidelines:
    \begin{itemize}
        \item The answer NA means that the abstract and introduction do not include the claims made in the paper.
        \item The abstract and/or introduction should clearly state the claims made, including the contributions made in the paper and important assumptions and limitations. A No or NA answer to this question will not be perceived well by the reviewers. 
        \item The claims made should match theoretical and experimental results, and reflect how much the results can be expected to generalize to other settings. 
        \item It is fine to include aspirational goals as motivation as long as it is clear that these goals are not attained by the paper. 
    \end{itemize}

\item {\bf Limitations}
    \item[] Question: Does the paper discuss the limitations of the work performed by the authors?
    \item[] Answer: \answerYes{} % Replace by \answerYes{}, \answerNo{}, or \answerNA{}.
    \item[] Justification: We address the limitations of our methods in Section~\ref{sec:discussion}.
    \item[] Guidelines:
    \begin{itemize}
        \item The answer NA means that the paper has no limitation while the answer No means that the paper has limitations, but those are not discussed in the paper. 
        \item The authors are encouraged to create a separate "Limitations" section in their paper.
        \item The paper should point out any strong assumptions and how robust the results are to violations of these assumptions (e.g., independence assumptions, noiseless settings, model well-specification, asymptotic approximations only holding locally). The authors should reflect on how these assumptions might be violated in practice and what the implications would be.
        \item The authors should reflect on the scope of the claims made, e.g., if the approach was only tested on a few datasets or with a few runs. In general, empirical results often depend on implicit assumptions, which should be articulated.
        \item The authors should reflect on the factors that influence the performance of the approach. For example, a facial recognition algorithm may perform poorly when image resolution is low or images are taken in low lighting. Or a speech-to-text system might not be used reliably to provide closed captions for online lectures because it fails to handle technical jargon.
        \item The authors should discuss the computational efficiency of the proposed algorithms and how they scale with dataset size.
        \item If applicable, the authors should discuss possible limitations of their approach to address problems of privacy and fairness.
        \item While the authors might fear that complete honesty about limitations might be used by reviewers as grounds for rejection, a worse outcome might be that reviewers discover limitations that aren't acknowledged in the paper. The authors should use their best judgment and recognize that individual actions in favor of transparency play an important role in developing norms that preserve the integrity of the community. Reviewers will be specifically instructed to not penalize honesty concerning limitations.
    \end{itemize}

\item {\bf Theory Assumptions and Proofs}
    \item[] Question: For each theoretical result, does the paper provide the full set of assumptions and a complete (and correct) proof?
    \item[] Answer: \answerYes{} % Replace by \answerYes{}, \answerNo{}, or \answerNA{}.
    \item[] Justification: We provide proofs of all main theorems in Appendix~\ref{thm:data-shapley} and Appendix~\ref{app:convergence}.
    \item[] Guidelines:
    \begin{itemize}
        \item The answer NA means that the paper does not include theoretical results. 
        \item All the theorems, formulas, and proofs in the paper should be numbered and cross-referenced.
        \item All assumptions should be clearly stated or referenced in the statement of any theorems.
        \item The proofs can either appear in the main paper or the supplemental material, but if they appear in the supplemental material, the authors are encouraged to provide a short proof sketch to provide intuition. 
        \item Inversely, any informal proof provided in the core of the paper should be complemented by formal proofs provided in appendix or supplemental material.
        \item Theorems and Lemmas that the proof relies upon should be properly referenced. 
    \end{itemize}

    \item {\bf Experimental Result Reproducibility}
    \item[] Question: Does the paper fully disclose all the information needed to reproduce the main experimental results of the paper to the extent that it affects the main claims and/or conclusions of the paper (regardless of whether the code and data are provided or not)?
    \item[] Answer: \answerYes{} % Replace by \answerYes{}, \answerNo{}, or \answerNA{}.
    \item[] Justification: We believe we have provided enough detail to reproduce our results in the main body and Appendix~\ref{app:setup}. We will provide all the code used in the experiments. 
    \item[] Guidelines:
    \begin{itemize}
        \item The answer NA means that the paper does not include experiments.
        \item If the paper includes experiments, a No answer to this question will not be perceived well by the reviewers: Making the paper reproducible is important, regardless of whether the code and data are provided or not.
        \item If the contribution is a dataset and/or model, the authors should describe the steps taken to make their results reproducible or verifiable. 
        \item Depending on the contribution, reproducibility can be accomplished in various ways. For example, if the contribution is a novel architecture, describing the architecture fully might suffice, or if the contribution is a specific model and empirical evaluation, it may be necessary to either make it possible for others to replicate the model with the same dataset, or provide access to the model. In general. releasing code and data is often one good way to accomplish this, but reproducibility can also be provided via detailed instructions for how to replicate the results, access to a hosted model (e.g., in the case of a large language model), releasing of a model checkpoint, or other means that are appropriate to the research performed.
        \item While NeurIPS does not require releasing code, the conference does require all submissions to provide some reasonable avenue for reproducibility, which may depend on the nature of the contribution. For example
        \begin{enumerate}
            \item If the contribution is primarily a new algorithm, the paper should make it clear how to reproduce that algorithm.
            \item If the contribution is primarily a new model architecture, the paper should describe the architecture clearly and fully.
            \item If the contribution is a new model (e.g., a large language model), then there should either be a way to access this model for reproducing the results or a way to reproduce the model (e.g., with an open-source dataset or instructions for how to construct the dataset).
            \item We recognize that reproducibility may be tricky in some cases, in which case authors are welcome to describe the particular way they provide for reproducibility. In the case of closed-source models, it may be that access to the model is limited in some way (e.g., to registered users), but it should be possible for other researchers to have some path to reproducing or verifying the results.
        \end{enumerate}
    \end{itemize}

\item {\bf Open access to data and code}
    \item[] Question: Does the paper provide open access to the data and code, with sufficient instructions to faithfully reproduce the main experimental results, as described in supplemental material?
    \item[] Answer: \answerYes{} % Replace by \answerYes{}, \answerNo{}, or \answerNA{}.
    \item[] Justification: In the supplemental material, we provide a simplified code example and will release the full codebase in a public code repository upon acceptance.
    \item[] Guidelines:
    \begin{itemize}
        \item The answer NA means that paper does not include experiments requiring code.
        \item Please see the NeurIPS code and data submission guidelines (\url{https://nips.cc/public/guides/CodeSubmissionPolicy}) for more details.
        \item While we encourage the release of code and data, we understand that this might not be possible, so “No” is an acceptable answer. Papers cannot be rejected simply for not including code, unless this is central to the contribution (e.g., for a new open-source benchmark).
        \item The instructions should contain the exact command and environment needed to run to reproduce the results. See the NeurIPS code and data submission guidelines (\url{https://nips.cc/public/guides/CodeSubmissionPolicy}) for more details.
        \item The authors should provide instructions on data access and preparation, including how to access the raw data, preprocessed data, intermediate data, and generated data, etc.
        \item The authors should provide scripts to reproduce all experimental results for the new proposed method and baselines. If only a subset of experiments are reproducible, they should state which ones are omitted from the script and why.
        \item At submission time, to preserve anonymity, the authors should release anonymized versions (if applicable).
        \item Providing as much information as possible in supplemental material (appended to the paper) is recommended, but including URLs to data and code is permitted.
    \end{itemize}

\item {\bf Experimental Setting/Details}
    \item[] Question: Does the paper specify all the training and test details (e.g., data splits, hyperparameters, how they were chosen, type of optimizer, etc.) necessary to understand the results?
    \item[] Answer: \answerYes{} % Replace by \answerYes{}, \answerNo{}, or \answerNA{}.
    \item[] Justification: We provide details of our experimental setup in Appendix~\ref{app:setup}.
    \item[] Guidelines:
    \begin{itemize}
        \item The answer NA means that the paper does not include experiments.
        \item The experimental setting should be presented in the core of the paper to a level of detail that is necessary to appreciate the results and make sense of them.
        \item The full details can be provided either with the code, in appendix, or as supplemental material.
    \end{itemize}

\item {\bf Experiment Statistical Significance}
    \item[] Question: Does the paper report error bars suitably and correctly defined or other appropriate information about the statistical significance of the experiments?
    \item[] Answer: \answerNo{} % Replace by \answerYes{}, \answerNo{}, or \answerNA{}.
    \item[] Justification: We report all results averaged over 100 random test trials. We avoid plotting error bars for some of the figures to avoid overlapping too many lines. 
    \item[] Guidelines:
    \begin{itemize}
        \item The answer NA means that the paper does not include experiments.
        \item The authors should answer "Yes" if the results are accompanied by error bars, confidence intervals, or statistical significance tests, at least for the experiments that support the main claims of the paper.
        \item The factors of variability that the error bars are capturing should be clearly stated (for example, train/test split, initialization, random drawing of some parameter, or overall run with given experimental conditions).
        \item The method for calculating the error bars should be explained (closed form formula, call to a library function, bootstrap, etc.)
        \item The assumptions made should be given (e.g., Normally distributed errors).
        \item It should be clear whether the error bar is the standard deviation or the standard error of the mean.
        \item It is OK to report 1-sigma error bars, but one should state it. The authors should preferably report a 2-sigma error bar than state that they have a 96\% CI, if the hypothesis of Normality of errors is not verified.
        \item For asymmetric distributions, the authors should be careful not to show in tables or figures symmetric error bars that would yield results that are out of range (e.g. negative error rates).
        \item If error bars are reported in tables or plots, The authors should explain in the text how they were calculated and reference the corresponding figures or tables in the text.
    \end{itemize}

\item {\bf Experiments Compute Resources}
    \item[] Question: For each experiment, does the paper provide sufficient information on the computer resources (type of compute workers, memory, time of execution) needed to reproduce the experiments?
    \item[] Answer: \answerYes{} % Replace by \answerYes{}, \answerNo{}, or \answerNA{}.
    \item[] Justification: We conduct empirical runtime benchmarking in Figure~\ref{fig:runtime}. We provide details of the hardware used for the experiments in Appendix~\ref{app:setup}.
    \item[] Guidelines:
    \begin{itemize}
        \item The answer NA means that the paper does not include experiments.
        \item The paper should indicate the type of compute workers CPU or GPU, internal cluster, or cloud provider, including relevant memory and storage.
        \item The paper should provide the amount of compute required for each of the individual experimental runs as well as estimate the total compute. 
        \item The paper should disclose whether the full research project required more compute than the experiments reported in the paper (e.g., preliminary or failed experiments that didn't make it into the paper). 
    \end{itemize}
    
\item {\bf Code Of Ethics}
    \item[] Question: Does the research conducted in the paper conform, in every respect, with the NeurIPS Code of Ethics \url{https://neurips.cc/public/EthicsGuidelines}?
    \item[] Answer: \answerYes{} % Replace by \answerYes{}, \answerNo{}, or \answerNA{}.
    \item[] Justification: We have reviewed the Code of Ethics and affirm that this paper conforms to these principles. 
    \item[] Guidelines:
    \begin{itemize}
        \item The answer NA means that the authors have not reviewed the NeurIPS Code of Ethics.
        \item If the authors answer No, they should explain the special circumstances that require a deviation from the Code of Ethics.
        \item The authors should make sure to preserve anonymity (e.g., if there is a special consideration due to laws or regulations in their jurisdiction).
    \end{itemize}

\item {\bf Broader Impacts}
    \item[] Question: Does the paper discuss both potential positive societal impacts and negative societal impacts of the work performed?
    \item[] Answer: \answerYes{} % Replace by \answerYes{}, \answerNo{}, or \answerNA{}.
    \item[] Justification: We have included a section on the broader impacts of our work in Appendix~\ref{app:ethics}. 
    \item[] Guidelines:
    \begin{itemize}
        \item The answer NA means that there is no societal impact of the work performed.
        \item If the authors answer NA or No, they should explain why their work has no societal impact or why the paper does not address societal impact.
        \item Examples of negative societal impacts include potential malicious or unintended uses (e.g., disinformation, generating fake profiles, surveillance), fairness considerations (e.g., deployment of technologies that could make decisions that unfairly impact specific groups), privacy considerations, and security considerations.
        \item The conference expects that many papers will be foundational research and not tied to particular applications, let alone deployments. However, if there is a direct path to any negative applications, the authors should point it out. For example, it is legitimate to point out that an improvement in the quality of generative models could be used to generate deepfakes for disinformation. On the other hand, it is not needed to point out that a generic algorithm for optimizing neural networks could enable people to train models that generate Deepfakes faster.
        \item The authors should consider possible harms that could arise when the technology is being used as intended and functioning correctly, harms that could arise when the technology is being used as intended but gives incorrect results, and harms following from (intentional or unintentional) misuse of the technology.
        \item If there are negative societal impacts, the authors could also discuss possible mitigation strategies (e.g., gated release of models, providing defenses in addition to attacks, mechanisms for monitoring misuse, mechanisms to monitor how a system learns from feedback over time, improving the efficiency and accessibility of ML).
    \end{itemize}
    
\item {\bf Safeguards}
    \item[] Question: Does the paper describe safeguards that have been put in place for responsible release of data or models that have a high risk for misuse (e.g., pretrained language models, image generators, or scraped datasets)?
    \item[] Answer: \answerNA{} % Replace by \answerYes{}, \answerNo{}, or \answerNA{}.
    \item[] Justification: We believe this paper does not pose such risks as we do not release any new models or datasets.
    \item[] Guidelines:
    \begin{itemize}
        \item The answer NA means that the paper poses no such risks.
        \item Released models that have a high risk for misuse or dual-use should be released with necessary safeguards to allow for controlled use of the model, for example by requiring that users adhere to usage guidelines or restrictions to access the model or implementing safety filters. 
        \item Datasets that have been scraped from the Internet could pose safety risks. The authors should describe how they avoided releasing unsafe images.
        \item We recognize that providing effective safeguards is challenging, and many papers do not require this, but we encourage authors to take this into account and make a best faith effort.
    \end{itemize}

\item {\bf Licenses for existing assets}
    \item[] Question: Are the creators or original owners of assets (e.g., code, data, models), used in the paper, properly credited and are the license and terms of use explicitly mentioned and properly respected?
    \item[] Answer: \answerYes{} % Replace by \answerYes{}, \answerNo{}, or \answerNA{}.
    \item[] Justification: We use publicly available datasets in our experiments and have cited the corresponding references. 
    \item[] Guidelines:
    \begin{itemize}
        \item The answer NA means that the paper does not use existing assets.
        \item The authors should cite the original paper that produced the code package or dataset.
        \item The authors should state which version of the asset is used and, if possible, include a URL.
        \item The name of the license (e.g., CC-BY 4.0) should be included for each asset.
        \item For scraped data from a particular source (e.g., website), the copyright and terms of service of that source should be provided.
        \item If assets are released, the license, copyright information, and terms of use in the package should be provided. For popular datasets, \url{paperswithcode.com/datasets} has curated licenses for some datasets. Their licensing guide can help determine the license of a dataset.
        \item For existing datasets that are re-packaged, both the original license and the license of the derived asset (if it has changed) should be provided.
        \item If this information is not available online, the authors are encouraged to reach out to the asset's creators.
    \end{itemize}

\item {\bf New Assets}
    \item[] Question: Are new assets introduced in the paper well documented and is the documentation provided alongside the assets?
    \item[] Answer: \answerNA{} % Replace by \answerYes{}, \answerNo{}, or \answerNA{}.
    \item[] Justification: This paper does not introduce new assets.
    \item[] Guidelines:
    \begin{itemize}
        \item The answer NA means that the paper does not release new assets.
        \item Researchers should communicate the details of the dataset/code/model as part of their submissions via structured templates. This includes details about training, license, limitations, etc. 
        \item The paper should discuss whether and how consent was obtained from people whose asset is used.
        \item At submission time, remember to anonymize your assets (if applicable). You can either create an anonymized URL or include an anonymized zip file.
    \end{itemize}

\item {\bf Crowdsourcing and Research with Human Subjects}
    \item[] Question: For crowdsourcing experiments and research with human subjects, does the paper include the full text of instructions given to participants and screenshots, if applicable, as well as details about compensation (if any)? 
    \item[] Answer: \answerNA{} % Replace by \answerYes{}, \answerNo{}, or \answerNA{}.
    \item[] Justification: This research does not use crowdsourcing with human subjects.b
    \item[] Guidelines:
    \begin{itemize}
        \item The answer NA means that the paper does not involve crowdsourcing nor research with human subjects.
        \item Including this information in the supplemental material is fine, but if the main contribution of the paper involves human subjects, then as much detail as possible should be included in the main paper. 
        \item According to the NeurIPS Code of Ethics, workers involved in data collection, curation, or other labor should be paid at least the minimum wage in the country of the data collector. 
    \end{itemize}

\item {\bf Institutional Review Board (IRB) Approvals or Equivalent for Research with Human Subjects}
    \item[] Question: Does the paper describe potential risks incurred by study participants, whether such risks were disclosed to the subjects, and whether Institutional Review Board (IRB) approvals (or an equivalent approval/review based on the requirements of your country or institution) were obtained?
    \item[] Answer: \answerNA{} % Replace by \answerYes{}, \answerNo{}, or \answerNA{}.
    \item[] Justification: This research does not use human subjects and is not subject to IRB approval. 
    \item[] Guidelines:
    \begin{itemize}
        \item The answer NA means that the paper does not involve crowdsourcing nor research with human subjects.
        \item Depending on the country in which research is conducted, IRB approval (or equivalent) may be required for any human subjects research. If you obtained IRB approval, you should clearly state this in the paper. 
        \item We recognize that the procedures for this may vary significantly between institutions and locations, and we expect authors to adhere to the NeurIPS Code of Ethics and the guidelines for their institution. 
        \item For initial submissions, do not include any information that would break anonymity (if applicable), such as the institution conducting the review.
    \end{itemize}

\end{enumerate}

\end{document}